\documentclass{article}

\PassOptionsToPackage{numbers, compress}{natbib}

\usepackage[main, final]{neurips_2026}
\usepackage{amsmath}

\usepackage[utf8]{inputenc} 
\usepackage[T1]{fontenc}    
\usepackage{xurl}           
\usepackage{booktabs, threeparttable, makecell}       
\usepackage{amsfonts}       
\usepackage{nicefrac}       
\usepackage{microtype}      
\usepackage{xcolor}         

\usepackage{graphicx}
\usepackage{amssymb}      
\usepackage{amsthm}       
\usepackage{xspace}
\usepackage{enumitem}
\usepackage{subfiles}

\graphicspath{{./}}

\newcommand{\parhead}[1]{\noindent\textbf{#1}.}
\newtheoremstyle{mytheoremstyle}{6pt}{0pt}{\itshape}{}{\bfseries}{.}{.5em}{} 
\theoremstyle{mytheoremstyle}
\newtheorem{theorem}{Theorem}[section]
\newtheorem{lemma}{Lemma}[section]
\newtheorem{assumption}{Assumption}[section]

\newtheorem{definition}{Definition}[section]
\newtheorem*{remark}{Remark}
\newcommand{\tool}{PANDA\xspace}
\newcommand{\base}{CROWN\xspace}

\newcommand{\R}{\mathbb{R}}
\newcommand{\Ball}{\mathbb{B}}
\newcommand{\preact}{z}
\newcommand{\postact}{h}
\newcommand{\prover}{\mathcal{P}}
\newcommand{\verifier}{\mathcal{V}}

\newcommand{\onevec}{\mathbf{1}}
\newcommand{\la}{\underline{a}}
\newcommand{\ua}{\overline{a}}
\newcommand{\lbof}{\underline{b}}
\newcommand{\ubof}{\overline{b}}
\newcommand{\lA}{\underline{\mathbf{A}}}
\newcommand{\uA}{\overline{\mathbf{A}}}
\newcommand{\ld}{\boldsymbol{\underline{d}}}
\newcommand{\ud}{\boldsymbol{\overline{d}}}

\newcommand{\Picom}{\Pi_{\textsc{com}}}
\newcommand{\Piarith}{\Pi_{\textsc{arith}}}
\newcommand{\Pilookup}{\Pi_{\textsc{lookup}}}

\newcommand{\bpreact}{\boldsymbol{z}}
\newcommand{\bpostact}{\boldsymbol{h}}
\newcommand{\bla}{\boldsymbol{\underline{a}}}
\newcommand{\bua}{\boldsymbol{\overline{a}}}
\newcommand{\blbof}{\boldsymbol{\underline{b}}}
\newcommand{\bubof}{\boldsymbol{\overline{b}}}

\usepackage{siunitx}
\usepackage{float}
\usepackage{hyperref}      
\usepackage{algorithm}     
\usepackage{algpseudocode}
\title{Certified but Private: Scalable Zero-Knowledge Proofs for Neural Network Guarantees}

\author{
  Youwei Zhong\\
  Yale University\\
  \texttt{youwei.zhong@yale.edu}\\
  \And
  Ben Merbaum\\
  Yale University\\
  \texttt{ben.merbaum@yale.edu}\\
  \And
  Timos Antonopoulos\\
  Yale University\\
  \texttt{timos.antonopoulos@yale.edu}\\
  \And
  Ning Luo\\
  University of Illinois Urbana-Champaign\\
  \texttt{nl27@illinois.edu}\\
  \And
  Charalampos Papamanthou\\
  Yale University\\
  \texttt{charalampos.papamanthou@yale.edu}\\
  \And
  Katerina Sotiraki\\
  Yale University\\
  \texttt{katerina.sotiraki@yale.edu}\\
  \And
  Ruzica Piskac\\
  Yale University\\
  \texttt{ruzica.piskac@yale.edu}\\
}

\begin{document}

\maketitle

\begin{abstract}
With the growing deployment of machine learning models, formal guarantees of the robustness and fairness of these models have become increasingly important in safety-critical and legal-compliance settings. However, model parameters are often commercial secrets that cannot be disclosed to auditors or end users. To this end, we present PANDA, a scalable system that uses zero-knowledge proofs (ZKPs) to prove the robustness and fairness properties of a model without revealing its private parameters. PANDA is built on top of CROWN, an efficient robustness certification framework that is used in many state-of-the-art formal verification tools for neural networks. The core contribution of PANDA is a novel algorithm for proving linear relaxation bounds for non-linear activation layers, yielding simple, lightweight proofs. Remarkably, our implementation can generate proofs of local robustness for neural networks with more than 2.9M parameters in 5 minutes, and can verify them in 10 seconds. Prior ZKP-based robustness system rely on exponential-time algorithms that cannot scale to nontrivial networks. In contrast, PANDA scales polynomially in the number of neurons in a network, allowing us to support neural networks 4 orders of magnitude larger than previous approaches with significantly reduced prover overhead.
\end{abstract}

\section{Introduction}
A machine learning (ML) model is \emph{robust} if small perturbations to its input do not cause large or unpredictable changes in its output. As ML systems are increasingly deployed in safety-critical settings such as healthcare, autonomous driving, and financial decision-making, scalable methods for formally verifying robustness are becoming ever more important. Beyond the model owners' own interest in reliability, regulators and downstream customers are increasingly demanding formal robustness guarantees~\cite{eu21, ca18}.
For example, a hospital may train a clinical model that predicts patient risk levels from reported symptoms. A user deciding whether to deploy the model naturally requires assurance that the model is robust: small perturbations to a patient's reported symptoms, such as a typo or synonym, should not change the patient's predicted risk level.

Formally, consider a classification model that takes as input a vector $\boldsymbol{x}$, which may encode an image, and outputs a predicted label. A standard notion of \textit{local robustness} requires that, for a given input $\boldsymbol{x}_0$, all sufficiently close (with respect to a parameter $
\varepsilon$) inputs receive the same prediction. This property can be formalized as:
\begin{equation}
\label{robustness}
\forall \boldsymbol{x}.\ \|\boldsymbol{x} - \boldsymbol{x}_0\| < \varepsilon \Rightarrow \mathrm{Classification}(\boldsymbol{x}) = \mathrm{Classification}(\boldsymbol{x}_0).
\end{equation}

One approach to proving local robustness properties is to provide users or third-party auditors with full access to the model parameters. However, modern ML models constitute valuable intellectual property. Networks are trained on proprietary datasets, and exposing model parameters can enable sophisticated white-box attacks~\cite{Nasr2018} and leak information about the training data~\cite{measuringdataleakage}. In the medical setting above, the hospital cannot disclose the model weights, as this may reveal protected health information about patients and violate HIPAA regulations.

The tension between verifiability and model confidentiality can be resolved by using \textit{zero-knowledge proofs} (ZKPs)~\cite{zkp}, a cryptographic primitive that enables a prover to convince a verifier that a statement holds without revealing any information beyond the statement's validity. A growing line of work called zkML has used ZKPs to hide model parameters in ML-related computation~\cite{zkGPT,deepprove}. However, most of it has focused on privacy-preserving ML inference and training integrity, leaving local robustness comparatively underexplored. Existing work that considers the related notion of fairness~\cite{fairproof} can only support models with $\approx$ 100 parameters, far too small for real-world applications.

Recent developments in formal verification have enabled fast robustness certification of deep neural networks (DNNs), but these tools reveal information about the model~\cite{neuralsat_proof, torchlean}. A natural idea is to compose a robustness certification algorithm with a ZKP to hide model parameters and inherit the efficiency and scalability of the underlying algorithm. However, this composition introduces several technical challenges: ZKPs natively operate with linear functions and must be adapted to handle non-linear activation functions. In addition, ZKPs are not typically compatible with floating-point arithmetic. Finally, performing computation within ZKPs introduces a significant overhead in the proving time, which can act as a barrier to scaling to large real-world DNNs. This leads us to the following problem statement:

\begin{center}
    \textit{Can we verify local robustness properties of large neural networks in zero-knowledge?}
\end{center}

We answer this question in the affirmative. In this paper, we propose \tool (Proofs of Automated Neural-network Derived Affine bounds), a system that combines ZKPs with the \base verification algorithm~\cite{crown} to produce publicly verifiable proofs of local robustness. \base is an efficient linear bound propagation method that underlies many state-of-the-art neural network verification systems~\cite{alphacrown, betacrown, neuralsat}. By integrating \base into a zero-knowledge framework, \tool enables model owners to certify robustness properties while keeping the underlying model parameters private.

To the best of our knowledge, \tool is the most efficient and scalable ZKP system for proving local robustness of private neural networks to date. Our implementation supports networks containing up to 2.9 million parameters, exceeding the scale of prior approaches by more than 4 orders of magnitude, while requiring a proving time of 5 minutes and a verification time of 10 seconds. Moreover, the prover runtime is polynomial in the number of neurons in the network, enabling practical certification for substantially larger models than previously possible.

\tool is also the first privacy-preserving local robustness certification system to support DNNs with transcendental activation functions such as sigmoid and tanh. We introduce a novel method for verifying linear relaxations of activation functions over continuous intervals by evaluating four pointwise inequalities, yielding a ZK-friendly check.

\parhead{Technical highlights} We propose a \textit{certification algorithm} which enables verifying the correctness of computation without performing the entire computation trace within a ZKP, drastically improving the efficiency of \tool. In particular, the \base algorithm requires finding a pair of linear lower and upper bounds that sandwich the activation function, which are selected through a costly iterative search~\cite{crown}. We observe that this search can be performed outside of the ZKP in order to reduce the prover overhead, and then the computed lower and upper bounds can be verified using a simple system of constraints, which we name the \textit{Four-Point Relaxation Gadget} (Section~\ref{sec:gadget}). This principle is a through-line of \tool: the prover first performs the \base algorithm outside of a ZKP, and then certifies its results through a reduced system of constraints within the ZKP. 

Instead of using a general-purpose ZKP backend, we design a \textit{customized backend} that make use of different cryptographic primitives for proving different operations. This enables \tool to achieve greater prover efficiency than previous work and scale to larger networks.

\section{Preliminaries}\label{sec:pre}
A \textbf{zero-knowledge proof} (ZKP) allows a prover $\prover$ to convince a verifier $\verifier$ that a public statement $x$ is true, without revealing any further information about \textit{why} it is true. Private information explaining the ``why'' is called the witness $w$. For instance, a public statement $x$ may be that an ML model is locally robust at a given point, whereas the private witness $w$ contains the private model weights and auxiliary values produced in computation. $\prover$ uses $x$ and $w$ to outputs a proof $\pi$, which $\verifier$ checks alongside $x$ and either accepts or rejects.

A \textbf{commit-and-prove} ZKP system enables $\prover$ to first produce a commitment $c_w$ to the witness $w$ which is \textit{hiding} ($c_w$ reveals nothing about $w$) and \textit{binding} ($\prover$ cannot find a different witness $w'\neq w$ corresponding to $c_w$). The commitment $c_w$ is part of the public statement $x$, and $\prover$'s claim is that $\prover$ knows a witness $w$ whose commitment is $c_w$ and that $w$ is a valid witness for $x.$

\tool is a commit-and-prove ZKP system which uses three underlying cryptographic primitives:
\begin{enumerate}
    \item A \textbf{polynomial commitment scheme}, $\Picom$, which consists of three algorithms:
    \begin{itemize}
        \item \textsc{Commit} (run by $\prover$) takes as input a polynomial $f$ and outputs a commitment $c_f$ to $f$ which is \textit{hiding} and \textit{binding}.
        \item \textsc{Eval} (run by $\prover$) takes as input $f$, a point $x$, and a value $y$, and outputs a proof $\pi$, also called an \textit{opening}, claiming that $y=f(x)$.
        \item \textsc{Verify} (run by $\verifier$) takes as input $c_f,x,y,\pi,$ and either accepts or rejects.
    \end{itemize}
    $\Picom$ is \textit{complete} if \textsc{Verify} accepts whenever $c_f\leftarrow \textsc{Commit}(f)$, $y=f(x)$, and $\pi\leftarrow \textsc{Eval}(f,x,y)$. We say $\Picom$ is \textit{evaluation-binding} if a malicious $\prover$ cannot produce an accepting proof $\pi$ for $x,y$ if $y\neq f(x)$. We say $\Picom$ is \textit{zero-knowledge} if $\pi$ reveals no information about $f.$
    \item A \textbf{ZKP for matrix arithmetic}, $\Piarith$.
    
    We can encode matrices or vectors as polynomials using interpolation and then use $\Picom$ for commitments to matrices or vectors.
    \begin{itemize}
    \item The private witness $w$ is a list of matrices (e.g., $\mathbf{A,B,C,D}$).
    \item The public statement $x$ is a list of commitments to these matrices (e.g. $c_A,c_B,c_C,c_D$) and a claimed arithmetic relation between them (e.g. $\mathbf{A=BC+D}$).
    \item $\prover$ proves to $\verifier$ that $\prover$ knows matrices in $w$ whose commitments match those in $x$ and satisfy this matrix arithmetic equality.
    \end{itemize}
    \item A \textbf{ZKP for table lookups}, $\Pilookup$ (also known as a \textit{lookup argument}).
    \begin{itemize}
        \item The private witness $w$ is a \textit{lookup vector} $\boldsymbol{a}$.\footnote{We present a simplified abstract interface. The concrete lookup arguments we use embed auxiliary information pertaining to $\boldsymbol{a}$ and $\boldsymbol{t}$ in the witness $w$ which is committed and appended to the public statement $x$.}
        \item The public statement $x$ is a commitment $c_a$ to $\boldsymbol{a}$ as well as a \textit{table vector} $\boldsymbol{t}$.
        \item $\prover$ proves to $\verifier$ that $\prover$ knows a vector $\boldsymbol{a}$ whose commitment is $c_a$ and such that each entry of $\boldsymbol{a}$ is in the table $\boldsymbol{t}$.
    \end{itemize}
\end{enumerate}
$\Piarith$ and $\Pilookup$ each contain two algorithms:
    \begin{itemize}
        \item \textsc{Prove} (run by $\prover$) takes as input $w$ and $x$ and outputs a proof $\pi$. As a subroutine, $\prover$ produces openings of masked versions of the witnesses via $\Picom.\textsc{Eval}$.
        \item \textsc{Verify} (run by $\verifier$) takes as input $x$ and $\pi$ and accepts or rejects. As a subroutine, $\verifier$ calls $\Picom.\textsc{Verify}$ to check the opening proofs (which are contained in $\pi$).
    \end{itemize}
We say that a ZKP is \textit{complete} if, whenever $w$ is a valid witness for $x$, $\textsc{Verify}$ accepts the proof $\pi$ generated by $\textsc{Prove}$. A ZKP is \textit{sound} if whenever $x$ is an invalid statement, no computationally bounded prover can forge a proof $\pi$ that causes $\textsc{Verify}$ to accept. A ZKP is \textit{zero-knowledge} if $\pi$ reveals no information about $w$ beyond the fact that $w$ exists.

$\Pilookup$ can be used for \textbf{range proof} ZKPs, which prove that each element of a vector $\boldsymbol{a}$ lies within an interval $[l,r].$ This is done simply by setting the table vector $\boldsymbol{t}=(l,l+1,\dots,r).$ We often write one-sided range proofs as $\boldsymbol{a}\geq 0$, where $l=0$ and $r$ is chosen implicitly as the maximum attainable value of elements in $\boldsymbol{a}$, dependent on the quantization.

$\Pilookup$ can also prove \textbf{non-linear function evaluations}; that is, for a set of points $\{(x_i,y_i)\}_{i\in[n]}$ and a non-linear function $f$, $\Pilookup$ proves that $f(x_i)=y_i$ for each $i\in[n]$. This is achieved by setting $\boldsymbol{a}=\left((x_i,y_i)\right)_{i\in[n]}$ and $\boldsymbol{t}=\left((z,f(z))\right)_{z}$, where $z$ ranges over possible values in the domain.

\parhead{Quantization}
Most existing ZKPs are incompatible with real-number computation. Following precedent in prior work~\cite{zkGPT,deepprove}, we perform \textit{quantization} to translate computations to finite field arithmetic. Following the approach of~\cite{8578384}, we represent each real number $x$ with a $Q$-bit integer $q_x\in [0,2^Q)$ satisfying $x\cdot S_x\approx q_x$ for an optimally chosen scaling factor $S_x\gg 1.$

Per-layer matrices and tensors in the neural network typically share a single scaling factor, and model weights and biases are all initialized with it. To sum two values $x$ and $y$, we first rescale them to share the same scaling factor, then sum the quantized integers $q_x$ and $q_y$. To multiply two values $x$ and $y$, we multiply the quantized integers $q_x$ and $q_y$, and set the scaling factor of the product to $S_xS_y$. To rescale $q_x\approx x\cdot S_x$ to some new scaling factor $S_z$, we compute $q_z=\left\lfloor S_z q_x/S_x\right\rfloor.$

\section{\base Algorithm}
\label{sec:crown}
In this section, we review the \base algorithm~\cite{crown}, a crucial building block for our \tool zero-knowledge protocol. We present each step of the algorithm in a way that enables compatibility with ZKP systems. The \base algorithm is provided in full detail in Appendix~\ref{app:crown-details:pseudocode}.

\subsection{Setting and Goal}
\label{sec:crown:setting}

We consider an $m$-layer feed-forward network $f: \R^{n_0} \to \R^{n_m}$ defined by the recursion
\begin{equation}\label{eq:network}
   \bpreact^{(k)} = \mathbf{W}^{(k)} \bpostact^{(k-1)} + \boldsymbol{b}^{(k)}\in\R^{n_k}, \qquad \bpostact^{(k)} = \sigma\!\left(\bpreact^{(k)}\right)\in\R^{n_k}
\end{equation}
for $k = 1, \dots, m$, with weight matrices $\mathbf{W}^{(k)} \in \R^{n_k \times n_{k-1}}$, biases $\boldsymbol{b}^{(k)} \in \R^{n_k}$, and a coordinate-wise activation function $\sigma:\R^{n_k}\rightarrow \R^{n_k}$ for all $k$. The input is $\bpostact^{(0)} = \boldsymbol{x}\in\R^{n_0}$ and the output is $f(\boldsymbol{x}) =\bpreact^{(m)}\in\R^{n_m}.$ The $\ell^\mathrm{th}$ layer has $n_\ell$ neurons, and the output neurons on input $\boldsymbol{x}_0$ are denoted $f_c(\boldsymbol{x}_0)$ for $c\in \{1,\dots,n_m\}.$ For a classification model, the choice of $c$ which maximizes $f_c(\boldsymbol{x}_0)$ is the \emph{predicted class} of $\boldsymbol{x}_0.$

Given an input $\boldsymbol{x}_0$, an $\ell^\infty$-perturbation radius $\varepsilon$, and a predicted class $c$ of $\boldsymbol{x}_0$, the network is \emph{locally robust} at $\boldsymbol{x}_0$ with respect to a targeted attack class $t \neq c$ if
\begin{equation}\label{eq:robust}  
f_t(\boldsymbol{x})\le f_c(\boldsymbol{x}), \qquad\forall \boldsymbol{x}\in \Ball(\boldsymbol{x}_0,\varepsilon),
\end{equation}
where $\Ball(\boldsymbol{x}_0, \varepsilon) := \{\boldsymbol{x} : \|\boldsymbol{x} - \boldsymbol{x}_0\|_\infty \le \varepsilon\}$.
Such a claim is a special case of the general statement
\begin{equation}\label{eq:generalized-robust}
    \mathbf{C}\cdot f(\boldsymbol{x})\le \boldsymbol{u}, \qquad \forall \boldsymbol{x}\in \Ball(\boldsymbol{x}_0,\varepsilon),
\end{equation}
where $\mathbf{C}\in \R^{N\times n_m},\boldsymbol{u}\in \R^{N}$, obtained by setting 
$N=1,\boldsymbol{u}=0$, and $\mathbf{C}=\boldsymbol{e}_t^{\top}-\boldsymbol{e}_c^{\top}$, where $\boldsymbol{e}_i$ is the $i^\mathrm{th}$ basis vector.

\parhead{Goal of the algorithm}
\base certifies statements of the form~\eqref{eq:generalized-robust} by running an algorithm which computes a vector $\boldsymbol{u}'$ such that $\mathbf{C}\cdot f(\boldsymbol{x})\leq \boldsymbol{u}'$, and then ultimately checking that $\boldsymbol{u}'\le \boldsymbol{u}$ entry-wise. Crucially, this algorithm computes \emph{linear bounds} on the pre-activation vector $\bpreact^{(k)}$ for each layer $k$ as a function of the input $\boldsymbol{x}$:
\begin{equation}\label{eq:layer-bound}
  \lA^{(k)} \boldsymbol{x} + \ld^{(k)} \;\le\; \bpreact^{(k)} \;\le\; \uA^{(k)} \boldsymbol{x} + \ud^{(k)}, \qquad \forall \boldsymbol{x} \in \Ball(\boldsymbol{x}_0, \varepsilon),
\end{equation}
where $\lA^{(k)}, \uA^{(k)} \in \R^{n_k \times n_0}$ and $\ld^{(k)}, \ud^{(k)} \in \R^{n_k}$. In the final layer $m$, \base constructs linear bounds on $\mathbf{C}\cdot \bpreact^{(m)}$ rather than on $\bpreact^{(m)}$.

\parhead{Outer forward pass, inner backward pass}
\base builds~\eqref{eq:layer-bound} layer by layer, passing forward from $k = 1$ to $k = m$. Computing the bounds at layer $k$, however, requires a \textit{backward pass} from layer $k$ down to the input, using the previously computed bounds for layers $\ell < k$. Once this is completed, a \textit{concretization} step transforms~\eqref{eq:layer-bound} into scalar bounds used to derive $\boldsymbol{u}'.$

\subsection{Backward Pass}
\label{sec:crown:backward}
The backward pass for layer $k$ begins with a set of concretized bounds $\{\boldsymbol{L}^{(\ell)}, \boldsymbol{U}^{(\ell)}\}_{\ell < k}$ for layers $\ell<k$, where each pre-activation vector $\bpreact^{(\ell)}$ satisfies the entry-wise inequality
\begin{equation}
  \boldsymbol{L}^{(\ell)} \;\le\; \bpreact^{(\ell)} \;\le\; \boldsymbol{U}^{(\ell)}, \qquad \forall \boldsymbol{x} \in \Ball(\boldsymbol{x}_0, \varepsilon).
\end{equation}
These are obtained by recursively calling \base for layers $\ell = 1,\dots,k-1$. The goal is to compute the linear bounds of~\eqref{eq:layer-bound} for layer $k$. We describe the upper bound case; the lower bound is symmetric.

\parhead{Activation relaxation}
A prerequisite is a pair of \emph{linear relaxations} of the activation $\sigma$. For every layer $\ell < k$, \base supplies four vectors $\bua^{(\ell)}, \bubof^{(\ell)}, \bla^{(\ell)}, \blbof^{(\ell)}$ such that
\begin{equation}\label{eq:relax}
  \bla^{(\ell)} \preact + \blbof^{(\ell)} \;\le\; \sigma(\bpreact^{(\ell)}) \;\le\; \bua^{(\ell)} \preact + \bubof^{(\ell)}, \qquad \forall \bpreact \in [\boldsymbol{L}^{(\ell)},\, \boldsymbol{U}^{(\ell)}].
\end{equation}
The linear relaxations are selected through an intricate procedure detailed in Appendix~\ref{app:crown-details:relaxation}.

\parhead{Invariants}
The backward pass maintains and updates an \emph{affine bound} $(\mathbf{A}, \boldsymbol{d})$ with $\mathbf{A} \in \R^{n_k \times n_\ell}$ and $\boldsymbol{d} \in \R^{n_k}$. At each depth $\ell$ from $k$ down to $0$, the bound satisfies one of two invariants:
\begin{align}
  \text{Invariant }\; \mathcal{I}_\postact^{(\ell)}\,:\quad &\bpreact^{(k)} \;\le\; \mathbf{A}\, \bpostact^{(\ell)} + \boldsymbol{d}, \label{eq:Ih} \\
  \text{Invariant }\; \mathcal{I}_\preact^{(\ell)}\,:\quad &\bpreact^{(k)} \;\le\; \mathbf{A}\, \bpreact^{(\ell)} + \boldsymbol{d}, \label{eq:Iz}
\end{align}
where the inequality is entry-wise and the bound holds for all $\boldsymbol{x} \in \Ball(\boldsymbol{x}_0, \varepsilon)$. The recursion starts at $\mathcal{I}_\postact^{(k-1)}$, ends with~\eqref{eq:layer-bound} at $\mathcal{I}_\postact^{(0)},$ and alternates between the two invariants as follows:
\begin{itemize}
  \item \textbf{Step 1 at depth $\ell$} transforms $\mathcal{I}_\postact^{(\ell)} \Rightarrow \mathcal{I}_\preact^{(\ell)}$ by linear relaxation through the activation.
  \item \textbf{Step 2 at depth $\ell$} transforms $\mathcal{I}_\preact^{(\ell)} \Rightarrow \mathcal{I}_\postact^{(\ell-1)}$ by substitution through the linear map.
\end{itemize}

\parhead{Initialization}
$\bpreact^{(k)} = \mathbf{W}^{(k)} \bpostact^{(k-1)} + \boldsymbol{b}^{(k)}$ establishes $\mathcal{I}_\postact^{(k-1)}$ by setting $\mathbf{A} \gets \mathbf{W}^{(k)}$ and $\boldsymbol{d} \gets \boldsymbol{b}^{(k)}.$

\parhead{Step 1: $\mathcal{I}_\postact^{(\ell)} \Rightarrow \mathcal{I}_\preact^{(\ell)}$}
We substitute $\bpostact^{(\ell)} = \sigma(\bpreact^{(\ell)})$ into~\eqref{eq:Ih} and apply the linear relaxation~\eqref{eq:relax}. Each entry of $\mathbf{A}$ uses the upper relaxation $(\ua^{(\ell)}_i, \ubof^{(\ell)}_i)$ when positive and the lower relaxation $(\la^{(\ell)}_i, \lbof^{(\ell)}_i)$ when negative, since multiplying by a negative number flips the inequality. Splitting $\mathbf{A}$ into 
\begin{equation}\label{eq:back-plusminus}
    \mathbf{A}_+ =\max(\mathbf{A},0),\qquad\mathbf{A}_-=\min(\mathbf{A},0)
\end{equation} 
yields the following update:
\begin{equation}\label{eq:relax-update}
  \boldsymbol{d} \gets \boldsymbol{d} + \mathbf{A}_+ \bubof^{(\ell)} + \mathbf{A}_- \blbof^{(\ell)}, \qquad \mathbf{A} \gets \mathbf{A}_+ \mathrm{diag}(\bua^{(\ell)}) + \mathbf{A}_- \mathrm{diag}(\bla^{(\ell)})
\end{equation}
where $\mathrm{diag}(\boldsymbol{v})$ denotes the diagonal matrix with entries of the vector $\boldsymbol{v}$ on its diagonal. After the update, $(\mathbf{A}, \boldsymbol{d})$ satisfies $\mathcal{I}_\preact^{(\ell)}$.

\parhead{Step 2: $\mathcal{I}_\preact^{(\ell)} \Rightarrow \mathcal{I}_\postact^{(\ell-1)}$}
We substitute $\bpreact^{(\ell)} = \mathbf{W}^{(\ell)} \bpostact^{(\ell-1)} + \boldsymbol{b}^{(\ell)}$, so
\begin{equation}\label{eq:linear-update}
  \boldsymbol{d} \gets \boldsymbol{d} + \mathbf{A} \boldsymbol{b}^{(\ell)}, \qquad \mathbf{A} \gets \mathbf{A} \mathbf{W}^{(\ell)}.
\end{equation}
The affine bound $(\mathbf{A},\boldsymbol{d})$ has advanced one level and now satisfies $\mathcal{I}_\postact^{(\ell-1)}$.

\parhead{Termination}
Iterating down to $\ell = 1$ produces $\mathcal{I}_\postact^{(0)}$, which is exactly equation~\eqref{eq:layer-bound} since $\bpostact^{(0)} = \boldsymbol{x}$. We finally retrieve $\uA^{(k)} \gets \mathbf{A}$ and $\ud^{(k)} \gets \boldsymbol{d}$.

\parhead{Final layer} In the final layer $k=m$, the backward pass must be modified in order to support the claim $\mathbf{C}f(\boldsymbol{x})\le \boldsymbol{u}$. Notably, invariants $\mathcal{I}_\postact^{(\ell)}$ and $\mathcal{I}_\preact^{(\ell)}$ replace the left-hand side $\bpreact^{(m)}$ with $\mathbf{C}\bpreact^{(m)}$. The initialization step is given by  $\mathbf{A} \gets \mathbf{C}\cdot \mathbf{W}^{(m)}$ and $\boldsymbol{d} \gets \mathbf{C}\cdot \boldsymbol{b}^{(m)}$. Steps 1 and 2 in the backward pass remain the same, and the final concretization yields the desired claim.

\subsection{Concretization}
\label{sec:crown:concretize}
Once the backward pass is complete for layer $k$, concretization yields the vectors $\boldsymbol{L}^{(k)}, \boldsymbol{U}^{(k)} \in \R^{n_k}$ by maximizing and minimizing the affine bounds of~\eqref{eq:layer-bound} over $\boldsymbol{x} \in \Ball(\boldsymbol{x}_0, \varepsilon)$. The upper bound $\boldsymbol{U}^{(k)}$ (and symmetrically the lower bound $\boldsymbol{L}^{(k)}$) is derived as
\begin{equation}\label{eq:concretize}
\boldsymbol{U}^{(k)}=\max_{\boldsymbol{x} \in \Ball(\boldsymbol{x}_0, \varepsilon)}\left (\uA^{(k)} \boldsymbol{x} + \ud^{(k)}\right) = \uA^{(k)}_+(\boldsymbol{x}_0 + \varepsilon\onevec) + \uA^{(k)}_-(\boldsymbol{x}_0 - \varepsilon\onevec) + \ud^{(k)},
\end{equation}
where $\onevec \in \R^{n_0}$ is the all-ones vector and 
\begin{equation}\label{eq:plusminus}
\uA^{(k)}_+=\max(\uA^{(k)},0),\qquad \uA^{(k)}_-=\min(\uA^{(k)},0).
\end{equation}

\section{\tool Framework}
\label{sec:panda}
This section presents \tool, our commit-and-prove ZKP for certifying local robustness of neural networks. We describe the steps used by $\prover$ to convince $\verifier$ of a local robustness claim in zero knowledge. The \tool algorithm is provided in full detail with pseudocode in Appendix~\ref{app:panda-details}.

\parhead{Setup}
$\prover$ and $\verifier$ agree on the input point $\boldsymbol{x}_0 \in \R^{n_0}$, the perturbation bound $\varepsilon$, the model architecture, and a public matrix $\mathbf{C}$ and vector $\boldsymbol{u}$ which define the local robustness claim of~\eqref{eq:generalized-robust},
\[
  \forall \boldsymbol{x} \in \Ball(\boldsymbol{x}_0, \varepsilon), \qquad \mathbf{C} f(\boldsymbol{x}) \le \boldsymbol{u}.
\]

\parhead{Initial pass}
$\prover$ runs quantized \base on input $(\boldsymbol{x}_0, \varepsilon)$, producing a \textit{transcript} for each layer $k$:
\begin{enumerate}
    \item[(a)] linear bounds $\lA^{(k)}, \uA^{(k)}, \ld^{(k)}, \ud^{(k)}$ of~\eqref{eq:layer-bound},
    \item[(b)] all intermediate states $(\mathbf{A},\boldsymbol{d})$ derived in the backward pass of \eqref{eq:Ih} and \eqref{eq:Iz},
    \item[(c)] signed component matrices from equations~\eqref{eq:relax-update} and~\eqref{eq:concretize},
    \item[(d)] concretized scalar bounds $\boldsymbol{L}^{(k)}$ and $\boldsymbol{U}^{(k)}$ of~\eqref{eq:concretize},
    \item[(e)] per-neuron linear relaxation coefficients $\bua^{(k)}, \bubof^{(k)}, \bla^{(k)}, \blbof^{(k)}$ of~\eqref{eq:relax}, and
    \item[(f)] quantization scaling factors $S_x$ for each value $x$ in the transcript.
\end{enumerate}

\parhead{Commit}
$\prover$ calls $\Picom.\textsc{Commit}$ to commit to the model parameters and all quantities in the \base transcript.

\parhead{Prove} The ZKP private witness $w$ contains the model parameters $\{\mathbf{W}^{(k)}, \boldsymbol{b}^{(k)}\}_{k=1}^{m}$ and the \base transcript. The public statement $x$ contains all commitments to the witness values together with the robustness claim $(\boldsymbol{x}_0, \varepsilon, \mathbf{C}, \boldsymbol{u})$ and the model architecture. 

For each operation performed in the initial pass of quantized \base, $\prover$ produces a proof $\pi$ using one of two underlying ZKPs: $\Piarith$ for matrix arithmetic and $\Pilookup$ for table lookups.
\begin{itemize}
   \item $\Piarith$ is used for $\prover$ to prove correctness of equations~\eqref{eq:Ih},\eqref{eq:Iz},\eqref{eq:relax-update},\eqref{eq:linear-update}, and \eqref{eq:concretize},

    \item $\Pilookup$ is used for $\prover$ to prove correctness of equations~\eqref{eq:back-plusminus} and~\eqref{eq:plusminus}.
 
\end{itemize}

\parhead{Proof of activation relaxations}
Existing primitives $\Piarith$ and $\Pilookup$ are insufficient to prove equation~\eqref{eq:relax}. In particular, $\prover$ must prove at each neuron that its committed scalars $\ua,\ubof,\la,\lbof$ satisfy 
\[
\la z+\lbof\le \sigma(z)\le\ua z+\ubof
\]
over a \textit{real-valued interval} $[l,u]$. While pointwise checks are feasible using $\Pilookup$, such checks do
not provide soundness guarantees between points. We therefore create a novel component to prove such equations, the \textit{Four-Point Relaxation Gadget}, which is the subject of Section~\ref{sec:gadget}.

\parhead{Proof of rescaling} All quantization operations must also be associated with proofs so that $\verifier$ can check $\prover$ does not deviate from the protocol. As discussed in Section~\ref{sec:pre}, a value $x$ is represented as $q_x\approx x\cdot S_x$ where $q_x$ is the quantized integer and $S_x$ is a scaling factor, also an integer. Rescaling to a new scaling factor $S_z$ requires computing $q_z=\left\lfloor S_z q_x/S_x\right\rfloor$. $\prover$ retains the remainder $r\in [0,S_x)$ of the rounding operation as an auxiliary witness, and proves that $0\leq r<S_x$ using $\Pilookup$. Finally, $\prover$ proves the identity
\begin{equation}
S_z\cdot q_x
= S_x\cdot q_z+r
\end{equation}
using $\Piarith$, which asserts the correctness of the rounding operation.

\parhead{Full protocol} The pseudocode of the \tool algorithm is provided in Appendix~\ref{app:panda-details}. \tool inherits completeness, soundness, and zero knowledge from $\Picom,\Piarith$, and $\Pilookup$. Robustness of the original model additionally requires guarantees on the effect of quantization, for which we provide a heuristic assumption and evaluate empirically. Further details are provided in Appendix~\ref{app:finite-encodings}.

\parhead{Asymptotic runtime} The runtime of \base and \tool are both polynomial in the number of neurons in the network. For an $m$-layer neural network with $n$ neurons per layer, \base runs in time $O(m^2n^3)$ (see Appendix~\ref{app:crown-details}). \tool inherits the same runtime, since we can instantiate the protocol with ZKP systems and a commitment scheme that require linear prover time. This enables \tool to scale to significantly larger neural networks than prior works.

\section{The Four-Point Relaxation Gadget}
\label{sec:gadget}
A core challenge in converting \base to a ZKP system is verifying that per-neuron linear relaxations in~\eqref{eq:relax} are valid over a real interval. Our construction is designed such that $\verifier$ can check pointwise inequalities at only \textit{four points} in order to determine the validity of a linear relaxation. Throughout this section, we focus on a single neuron with pre-activation bounds $l \le u$ and fix an activation function $\sigma.$

\begin{definition}[Linear relaxation of an activation function] A linear relaxation for $\sigma$ over an interval $[l,u]$ is a pair of affine functions
\begin{equation}\label{eq:gadget:envelope}
  h_L(\preact) = \la \preact + \lbof, \qquad h_U(\preact) = \ua \preact + \ubof,
\end{equation}
satisfying
\begin{equation}\label{eq:gadget:validity}
  h_L(\preact) \;\le\; \sigma(\preact) \;\le\; h_U(\preact), \qquad \forall \preact \in [l, u].
\end{equation}
\end{definition}

Figure~\ref{fig:linear-relaxations} in Appendix~\ref{app:crown-details:relaxation} depicts examples of linear relaxations of the ReLU and $\tanh$ functions. 

We design a system such that $\verifier$ can verify whether~\eqref{eq:gadget:validity} holds by checking only a constant number of pointwise inequalities. Observe that verifying~\eqref{eq:gadget:validity} over all possible $z\in [l,u]$ is impossible (there are infinitely many such $z$), whereas checking~\eqref{eq:gadget:validity} at finitely many sample points does not guarantee that the equation holds between these sample points. 

In Section~\ref{sec:gadget:sshape}, we focus on S-shaped activation functions such as sigmoid $\left(\sigma(x)=\frac{e^x}{e^x+1}\right)$, tanh $\left(\sigma(x)=\frac{e^x-e^{-x}}{e^x+e^{-x}}\right)$, and arctan $\left(\sigma(x)=\tan^{-1}(x)\right)$, which arise naturally in ML applications. In Section~\ref{sec:gadget:relu}, we focus on ReLU activations $\left(\sigma(x)=\max(0,x)\right)$.

\subsection{S-shaped Activations}
\label{sec:gadget:sshape}
We first focus on S-shaped activations and we show that linear relaxation validity reduces to checks at only \textit{four} points per neuron, independent of interval width. 

\begin{assumption}[S-shaped activation]\label{ass:sshape}
$\sigma: \R \to \R$ is $C^1$, monotone increasing, bounded, and has a single inflection point at $\preact = 0$. Its derivative $\sigma'$ is strictly increasing on $(-\infty,0]$ and strictly decreasing on $[0,\infty)$, with $B:=\sigma'(0)>0$.
\end{assumption}

Sigmoid, tanh, and arctan satisfy Assumption~\ref{ass:sshape}. The assumption implies that $\sigma$ is strictly convex on $(-\infty,0]$ and strictly concave on $[0,\infty)$, and that the derivative has a finite maximum $B$, attained only at $\preact = 0$. We restrict to slopes $\la, \ua \in (0, B]$; the selection procedure is given in Appendix~\ref{app:crown-details:relaxation}.

\begin{lemma}\label{lemma:two_tangency_points}
Let $\sigma$ satisfy Assumption~\ref{ass:sshape}. For any $\alpha \in (0, B]$, there exist two tangency points $z_1 \le 0 \le z_2$ such that $\sigma'(z_1) = \sigma'(z_2) = \alpha$. These are the only solutions of $\sigma'(z)=\alpha$. Moreover, these points are distinct if $\alpha < B$, and they coincide at $z_1 = z_2 = 0$ if $\alpha = B$.
\end{lemma}

See proof in Appendix~\ref{proof:two_tangency_points}. We refer to the classical Interior Extremum Theorem from calculus in order to provide a simple set of criteria for the validity of candidate linear relaxations. The proof of the lemma below follows directly from applying this theorem to $h_U(z)-\sigma(z).$ 
\begin{lemma}[Four-point soundness]\label{lemma:fourpoint}
Let $\sigma$ satisfy Assumption~\ref{ass:sshape}, and fix an interval $[l, u] \subseteq \R$. The upper bound $h_U(\preact) = \ua \preact + \ubof$ with $\ua > 0$ satisfies $\sigma(\preact) \le h_U(\preact)$ for all $\preact \in [l, u]$ if and only if
\begin{enumerate}
  \item[(i)] \emph{Endpoints:} $\sigma(l) \le h_U(l)$ and $\sigma(u) \le h_U(u)$;
  \item[(ii)] \emph{Tangency points:} for every $\preact \in [l, u]$ with $\sigma'(\preact) = \ua$, $\sigma(\preact) \le h_U(\preact)$.
\end{enumerate}
The dual statement holds for the lower bound with the inequality directions reversed.
\end{lemma}
By Lemma~\ref{lemma:two_tangency_points}, $\sigma'$ has at most two preimages of $\ua$, so it suffices to verify the values of $\sigma$ and $h_U$ at no more than four points: the two endpoints $l, u$ and the two tangency points $z_1, z_2$.

\parhead{The case-work obstacle}
Encoding Lemma~\ref{lemma:fourpoint} directly into a ZKP requires case-work at each tangency point $z_i$ for $i\in \{1,2\}$: there must be a range proof asserting that either (a) $z_i \le l$, (b) $z_i \ge u$, or (c) $\sigma(z_i) \le h_U(z_i)$. Declaring which case holds for each $z_i$, however, leaks information about $l$ and $u$, which are themselves derived from the private model parameters. We therefore replace Lemma~\ref{lemma:fourpoint} with a \emph{case-free}, sufficient (but not necessary) set of criteria for determining whether $h_U$ and $h_L$ are a valid linear relaxation over the interval $[l,u].$

\begin{theorem}[Case-free four-point soundness]\label{thm:casefreefourpoint}
Let $\sigma$ satisfy Assumption~\ref{ass:sshape} and fix an interval $[l, u] \subseteq \R$. Let $\ua\in(0,B]$ and let $z_1 \le 0 \le z_2$ be the (possibly coinciding) tangency points satisfying $\sigma'(z_1) = \sigma'(z_2) = \ua$. The upper bound $h_U$ is valid on $[l,u]$ if each of the following criteria is satisfied.
\begin{itemize}
  \item[(i)] \emph{Endpoints:} $\sigma(l) \le h_U(l)$ and $\sigma(u) \le h_U(u)$;
  \item[(ii)] \emph{Left tangency:} $\sigma(z_1) \le h_U(z_1)$;
  \item[(iii)] \emph{Right tangency:} $0 \le (u - z_2)\bigl(h_U(z_2) - \sigma(z_2)\bigr)$;
\end{itemize}
 A symmetric statement holds for the lower bound $h_L$ and is given in Appendix~\ref{proof:fourpoint}. Moreover, the linear relaxations selected by \base (given in Appendix~\ref{app:crown-details:relaxation}) satisfy these criteria.
\end{theorem}
Appendix~\ref{proof:fourpoint} contains the proof and an accompanying illustration in Figure~\ref{fig:four-point-soundness}.

\parhead{ZKP Algorithm}
We present the Four-Point Relaxation Gadget (formal algorithm in Appendix~\ref{app:panda-details}) as a simple set of criteria which $\prover$ proves and $\verifier$ verifies. We encode Theorem~\ref{thm:casefreefourpoint} as a set of ZKP constraints, which invoke $\Piarith$ and $\Pilookup$ ZKP systems as subroutines. 

For each neuron, and for the upper bound case, $\prover$ takes as private witness the interval endpoints $(l,u)$, the upper bound $(\ua,\ubof)$, and the tangency points $(z_1,z_2)$. $\prover$ evaluates $\sigma$ at the witness points and proves the correctness of these evaluations using lookup tables $T_\sigma$ and $T_{\sigma'}$. $\prover$ then proves the conditions of Theorem~\ref{thm:casefreefourpoint} are met. Algorithm~\ref{alg:gadget} of Appendix~\ref{app:panda-details} formally presents the prover for the upper-bound case, and the lower-bound case is analogous. $\verifier$ checks the correctness of all proofs.

\subsection{ReLU Activations}
\label{sec:gadget:relu}

For ReLU, $\sigma(\preact) = \max(0, \preact)$ is piecewise linear and convex. \base selects lower and upper bounds $(h_L,h_U)$ for an interval $[l,u]$ as follows:
\begin{itemize}
    \item \textbf{Upper bound $h_U$:} select the secant line connecting $(l,\sigma(l))$ and $(u,\sigma(u))$, given by
    \[
    h_U(z)=\frac{\sigma(u)-\sigma(l)}{u-l}(z-l)+\sigma(l).
    \]
    \item \textbf{Lower bound $h_L:$} select $h_L(z)=\la z$ for any $\la\in\{0,1\}.$ Either choice of $a$ guarantees a valid linear relaxation. \citet{crown} recommends selecting $\la=\mathbf{1}\left[|u| \ge |l|\right].$ 
\end{itemize}
We design a simple algorithm for proving the validity of such a linear relaxation $(h_L,h_U).$
\begin{itemize}
    \item \textbf{Upper bound $h_U$:} prove $\sigma(l)=h_U(l)$ and $\sigma(u)=h_U(u).$
    \item \textbf{Lower bound $h_L$:} prove $\la(\la-1)=0$ (i.e. $\la\in\{0,1\}$) and intercept $\lbof=0.$
\end{itemize}
Each of these conditions is proven with $\Piarith.$ For further generality, $\la$ can be selected within the closed interval $[0,1]$ and $\Pilookup$ can be used for the proof, thus supporting $\alpha$-CROWN \cite{alphacrown}.

\section{Implementation and Evaluations}
\label{sec:evaluation}

We implement \tool in Rust, using the arkworks ecosystem~\cite{arkworks} for ZKP programming. We use Hyrax~\cite{hyrax} over the elliptic curve BN254 for $\Picom$, the protocol of~\citet{thaler13} for $\Piarith$, and LogUp-GKR~\cite{logup-GKR} for $\Pilookup$. Our implementation uses the non-zero-knowledge variants of these protocols for simplicity, but we estimate that the use of zero-knowledge variants would increase the prover time by only a small constant factor (e.g., see~\cite{hypernova,vega}). In our experiments, we use AMD EPYC 9575F processors. Each job evaluates one model and runs on one core with 72 GiB of maximum allocated memory. Our implementation is available at \url{https://github.com/youweizhong/PANDA}.

In our evaluations, we answer the following questions:
\textbf{Q1: Scalability.} How does \tool scale to large networks?
\textbf{Q2: Support of activation functions.} How does \tool perform on activations other than ReLU?
\textbf{Q3: Quantization error.} Is there any mismatch between certification results for \tool and \base?
\textbf{Q4: Comparison with \base.} How much is the cost paid for privacy?
\textbf{Q5: Comparison with previous work.} How does \tool perform compared with previous work?

\parhead{Benchmark suites} We evaluate \tool on the MNIST image classification models~\cite{lecun-mnist-1998} from the \base evaluation. In addition, we evaluate \tool on two application-focused benchmarks: SafeNLP~\cite{safenlp} from VNN-COMP 2025~\cite{vnncomp2025} includes medical models that classify medical queries as high or low risk, and robot-detection models that classify text as originating from a human or robot. LunarLander~\cite{lunarlander} from VNN-COMP 2022~\cite{vnncomp2022} classifies safe actions to be taken by lunar robots. Finally, we evaluate on the only publicly available benchmark from FairProof's suite~\cite{fairproof}. 

\parhead{Evaluations} We run 100 tests per model, except for FairProof (one test is run) and MNIST (the 100 images are filtered to those correctly classified by the model). Only images which \base first certifies are tested with \tool. We use 8 bits of quantization for FairProof and 14 bits for all other evaluations. Our results are summarized in Table~\ref{final_least_table}, with full benchmarks in Table~\ref{complete_final_least_table} and a component-wise breakdown of prover time in Appendix~\ref{prover_breakdown}.

\begin{table}
  \caption{Summary of results. Each row is one model. We denote an $m$-layer feed-forward network with $n$ neurons per hidden layer by $m\times [n].$ Prove and verify times are reported as mean and standard deviation. FairProof Adult is a two-layer network with 8 and 2 neurons, denoted $[8,2]$.}
  \label{final_least_table}
  \centering
  \small
  \setlength{\tabcolsep}{3.5pt}
  \begin{tabular}{@{}lllS[separate-uncertainty,table-format=3.2(3.1)]S[separate-uncertainty,table-format=2.2(2.1)]rrr@{}}
    \toprule
    Dataset & Structure & Act. & {Prove (s)} & {Verify (s)} & {Proof} & {\base} & {Drift}\\
            &           &      &             &              & {(MB)} & {(s)} & {(\%)}\\
    \midrule
    \multicolumn{8}{@{}l}{\textit{MNIST}} \\
     & $2 \times [20]$ & ReLU & 22.81(20) & 0.97(1) & 3 & 1.3e-4 & 0.02 \\
     & $3 \times [1024]$ & ReLU & 101.58(375) & 4.60(13) & 9 & 0.22 & 0.03 \\
     & $4 \times [1024]$ & ReLU & 226.08(1467) & 9.56(48) & 17 & 0.44 & -0.01 \\
     & $4 \times [1024]$ & Sigmoid & 327.47(1136) & 14.13(46) & 31 & 0.47 & 2.48  \\
     & $4 \times [1024]$ & Tanh & 333.94(581) & 14.31(28) & 32 & 0.46 & 6.33 \\
    \addlinespace
    \multicolumn{8}{@{}l}{\textit{Other}} \\
    SafeNLP\,med & $2 \times [128]$ & ReLU & 23.04(19) & 0.81(1) & 3 & 1.3e-4 & -0.05  \\
    SafeNLP\,robot & $2 \times [128]$ & ReLU & 20.33(19) & 0.78(1) & 3 & 1.4e-4 & -0.29 \\
    LunarLander & $3 \times [64]$ & ReLU & 36.61(29) & 1.46(2) & 5 & 8.1e-5 & -0.09 \\
    FairProof\,Adult & $[8, 2]$ & ReLU & 34.89 & 1.21 & 4 & 1.4e-5 & -12.60 \\
    \bottomrule
  \end{tabular}
\end{table}

\parhead{Q1: Scalability} As shown in Table~\ref{final_least_table}, \tool produces proofs of local robustness in under six minutes for the largest $4\times [1024]$ MNIST networks, which have 2.9 million parameters. These networks are 4 orders of magnitude larger than previous proofs of fairness from FairProof~\cite{fairproof}.

\parhead{Q2: Support of activation functions} Relative to ReLU networks, sigmoid/tanh increases proving time by $1.5$--$2.6\times$, with verification time and proof size growing by a similar factor. This is due to the increased number of checks present in the Four-Point Relaxation Gadget for non-ReLU activations.

\parhead{Q4: Quantization error} Quantization in \tool produces errors relative to exact arithmetic in \base, since rounding operations are lossy and lookup tables discretize the real interval. We assess the effect of quantization error empirically, by measuring the drift between \tool and \base in Table~\ref{final_least_table} and recording any differences in the certification results. Our results show that \tool verifies all claims that \base verifies, and that the drift is limited to small differences which can be managed by tuning the quantization precision. See Appendix~\ref{app:finite-encodings} for further discussion.

\parhead{Q4: Comparison with \base} The cost of privacy is a prover overhead of roughly $100$--$1{,}000\times$ and a verifier overhead of $20$--$50\times$ over \base on the 1024-neuron models (rising to $10^6\times$ only on toy networks where \base runs in microseconds and ZKP fixed costs dominate).

\parhead{Q5: Comparison with previous work} FairProof is the only suitable direct comparison; it is the only ZKP system that can certify local robustness, to the best of our knowledge. We evaluate FairProof head-to-head with \tool on small benchmarks. We observe that FairProof cannot run on even the smallest network in the MNIST benchmark. Where FairProof runs (on FairProof Adult), \tool wins on prover time, requiring about half as much as FairProof’s prover time, with a feasible verifier cost ($1.21$s) and proof size ($4$MB). FairProof has smaller proof size and verifier time, whereas \tool's advantage is prover feasibility and scale. See Appendix \ref{comparison} for a detailed discussion.

\section{Related Work}

\parhead{Proof of fairness} Fairness is a notion related to robustness which states that sensitive characteristics such as race and gender should not affect model outputs. FairProof~\citep{fairproof} introduces a framework for verifying local fairness and robustness properties using ZKPs. Their approach supports only ReLU activations and enumerates all possible neuron activations in a preprocessing phase. The prover time grows exponentially with the number of neurons, which does not scale well to large models. FairZK by \cite{fairzk} and OATH by \cite{franzese2025} prove a notion of \textit{global fairness} using aggregated statistics; this direction is disjoint from our work and produces weaker bounds. Confidential-PROFITT by \cite{confidentialprofitt} enables privacy-preserving proofs of training for tree-split constraints in decision trees. A separate line of work, including \cite{fairness_in_the_eyes_of_data,verifiable_fairness}, focuses on verifying global fairness in a third-party auditor setting.

\parhead{Non-cryptographic certification approaches} 
Prior work has developed powerful formal verification techniques for certifying local robustness of DNNs~\cite{reluplex,crown,alphacrown}. However, these approaches operate in a non-cryptographic setting, where the verifier requires access to model parameters. In contrast, \tool certifies such properties in zero knowledge, without revealing model parameters. 

\parhead{Concurrent work} \citet{song2026securecrown} introduce the SecureCROWN algorithm, which augments CROWN to add privacy and verifiability via secure two-party computation (2PC). As a result, a data owner can hide their query point $\boldsymbol{x}_0$ and perturbation radius $\varepsilon$ while the model owner hides the model. However, their computation trace is not publicly verifiable, and their construction is in the honest-but-curious setting which assumes the server does not deviate from the protocol. The use of 2PC results in high communication overhead between the two parties, and their protocol supports only ReLU activations and not sigmoid or tanh. Another concurrent work, Azkaban~\cite{azkaban}, performs ZKPs over the DeepZono abstract interpretation algorithm~\cite{deepzono}, and applies fixed-point interval analysis to ensure quantization provably retains soundness, while our results provide empirical guarantees. Compared to Azkaban, \tool enables faster proving time with substantially smaller proof sizes.

\section{Conclusion}
In this paper, we introduced the \tool protocol for confidentially proving and verifying local robustness of model outputs. We evaluated \tool on leading benchmarks and demonstrated our performance improvements in prover time and scalability. \tool promotes transparency in ML infrastructure by enabling model owners to publicly share verifiable proofs of robustness without revealing model parameters. We hope this work will encourage further research in this direction.

\parhead{Limitations}\label{sec:limitations}
\tool proves satisfaction of the encoded quantized certification
constraints and not robustness for floating-point models. This limitation is consistent with previous works in zkML~\cite{zkGPT,deepprove}. We evaluate the impact of quantization on the soundness of \tool empirically; provable soundness guarantees for quantized certification remains future work (see Appendix~\ref{app:finite-encodings}), potentially by adapting the sound interval analysis of Azkaban~\cite{azkaban}. Additionally, \tool discloses the model architecture and layer dimensions in the public statement, hiding only the model weights and biases. However, DNNs can be padded with dummy layers and weights to hide exact dimensions. From the formal verification side, complete verification of local robustness of DNNs is NP-hard~\cite{reluplex}. Algorithms such as \base are \textit{incomplete verifiers}: they can verify many but not all robustness claims, and cannot reliably falsify a claim. \tool inherits this property of being an incomplete verifier.


\label{sec:main-end}
\clearpage
\label{sec:backmatter-start}

\acksection
This material is based upon work supported in part by the Defense Advanced Research Projects Agency (DARPA) under Agreement No. HR00112590130 and by the NSF awards CCF-2219995, CNS-2245344, and CCF-2318974. Youwei Zhong is supported by a Graduate Student Fellowship from the Wu Tsai Institute. Youwei Zhong's attendance at the conference to present this work is supported by a discretionary fund from the Wu Tsai Institute and the NeurIPS 2026 Financial Assistance Award. Ben Merbaum gratefully acknowledges financial support for this project by the Fulbright Canada Student Program, which is sponsored by the U.S. Department of State and Fulbright Canada. Any opinions, findings, and conclusions or recommendations expressed in this material are those of the authors and do not necessarily reflect the views of the funding agencies. We thank all anonymous reviewers for their constructive comments. We also thank Tianyu Zhang for insightful discussions on implementing customized ZKP systems. We acknowledge the resources provided by the Yale Center for Research Computing.

\bibliography{bib} 
\bibliographystyle{abbrvnat}
\newpage
\appendix

\ifSubfilesClassLoaded{\appendix}{}

\section{Ethical and societal considerations}\label{sec:societal-impact}
PANDA supports independent verification of specified model properties while protecting model parameters. This can improve accountability when model owners cannot disclose their models to auditors or users. These benefits are limited by the scope of the certified statement. Unsound quantized certification does not guarantee local robustness of a quantized model. A local robustness certificate for a quantized model does not establish global fairness, overall safety, or robustness of an arbitrary floating-point deployment. Certificates could be presented selectively or interpreted more broadly than their formal guarantees justify. Keeping model parameters private can also limit other forms of inspection that a robustness certificate does not replace.

To reduce these risks, we state the soundness gap, certified property, quantization assumptions, and public information explicitly, and distinguish the proposed ZKP system from the evaluated non-zero-knowledge implementation. The implementation provides no zero-knowledge guarantee and should not be directly used to protect confidential model parameters.

\section{Details of \base}
\label{app:crown-details}

In this section, we provide the full descriptions of the \base algorithm and the \tool framework, with pseudocode algorithms.

\subsection{The \base Algorithm}
\label{app:crown-details:pseudocode}

Algorithm~\ref{alg:crown} gives the full \base procedure. This consists of an outer forward pass over layers $k$, and an inner backward pass from layer $k$ down to the input layer 0. The backward pass repeatedly calls \textsc{Step1}, the activation-relaxation update~\eqref{eq:relax-update}, and \textsc{Step2}, the linear-map substitution~\eqref{eq:linear-update}. Concretization~\eqref{eq:concretize} is applied following each backward pass. For simplicity of notation, we assume relaxation scalars $(\bua^{(\ell)},\bubof^{(\ell)},\bla^{(\ell)},\blbof^{(\ell)})$ for each layer $\ell$ are implicitly passed as input to all procedures below upon generation.

\begin{algorithm}[!htbp]
\caption{The \base algorithm. \textsc{Relax} produces linear relaxations of $\sigma$ (see Section~\ref{sec:gadget}, and the selection algorithm in Appendix~\ref{app:crown-details:relaxation}).}
\label{alg:crown}
\begin{algorithmic}[1]
\Require Network $\{\mathbf{W}^{(k)}, \boldsymbol{b}^{(k)}\}_{k=1}^m$, input $\boldsymbol{x}_0$, perturbation budget $\varepsilon$, robustness claim $(\mathbf{C},\boldsymbol{u})$

\For{$k = 1, \dots, m$} \Comment{outer forward pass}
  \State $(\lA^{(k)}, \ld^{(k)}) \gets \textsc{BackwardPass}(k,L)$
  \State $(\uA^{(k)}, \ud^{(k)}) \gets \textsc{BackwardPass}(k,U)$
  \State $(\boldsymbol{L}^{(k)}, \boldsymbol{U}^{(k)}) \gets \textsc{Concretize}(\lA^{(k)}, \ld^{(k)}, \uA^{(k)}, \ud^{(k)}, \boldsymbol{x}_0, \varepsilon)$
  \For{each neuron $i$ in layer $k<m$}
    \State $(\ua^{(k)}_i, \ubof^{(k)}_i, \la^{(k)}_i, \lbof^{(k)}_i) \gets \textsc{Relax}(\sigma, L^{(k)}_i, U^{(k)}_i)$
  \EndFor
\EndFor
\State \Return True if $\boldsymbol{U}^{(m)}\le \boldsymbol{u}$ entry-wise, else False otherwise
\Statex
\Procedure{BackwardPass}{$k,s\in\{L,U\}$}
\If{$k<m$}
  \State $\mathbf{A} \gets \mathbf{W}^{(k)}, \quad \boldsymbol{d} \gets \boldsymbol{b}^{(k)}$ \Comment{initializes $\mathcal{I}_\postact^{(k-1)}$}
\ElsIf{$k=m$}
 \State $\mathbf{A} \gets \mathbf{C}\mathbf{W}^{(k)}, \quad \boldsymbol{d} \gets \mathbf{C}\boldsymbol{b}^{(k)}$ \Comment{final-layer initialization}
 \EndIf
  \For{$\ell = k-1, \dots, 1$}
    \State $(\mathbf{A}, \boldsymbol{d}) \gets \textsc{Step1}(\mathbf{A}, \boldsymbol{d}, \ell,s)$
    \State $(\mathbf{A}, \boldsymbol{d}) \gets \textsc{Step2}(\mathbf{A}, \boldsymbol{d}, \ell)$ \Comment{$\mathcal{I}_\preact^{(\ell)} \Rightarrow \mathcal{I}_\postact^{(\ell-1)}$}
  \EndFor
  \State \Return $(\mathbf{A}, \boldsymbol{d})$ \Comment{termination at $\mathcal{I}_\postact^{(0)}$}
\EndProcedure
\Statex
\Procedure{Step1}{$\mathbf{A}, \boldsymbol{d}, \ell,s$}
  \State $\mathbf{A}_+ \gets \max(\mathbf{A}, 0)$, \quad $\mathbf{A}_- \gets \min(\mathbf{A}, 0)$ 
  \If{$s=U$}
  \State $\boldsymbol{d} \gets \boldsymbol{d} + \mathbf{A}_+\, \bubof^{(\ell)} + \mathbf{A}_-\, \blbof^{(\ell)}$
  \State $\mathbf{A} \gets \mathbf{A}_+\, \mathrm{diag}(\bua^{(\ell)}) + \mathbf{A}_-\, \mathrm{diag}(\bla^{(\ell)})$
  \Else
  \State $\boldsymbol{d} \gets \boldsymbol{d} + \mathbf{A}_+\, \blbof^{(\ell)} + \mathbf{A}_-\, \bubof^{(\ell)}$
  \State $\mathbf{A} \gets \mathbf{A}_+\, \mathrm{diag}(\bla^{(\ell)}) + \mathbf{A}_-\, \mathrm{diag}(\bua^{(\ell)})$
  \EndIf
  \State \Return $(\mathbf{A}, \boldsymbol{d})$
\EndProcedure
\Statex
\Procedure{Step2}{$\mathbf{A}, \boldsymbol{d}, \ell$} \Comment{linear-map substitution~\eqref{eq:linear-update}}
  \State $\boldsymbol{d} \gets \boldsymbol{d} + \mathbf{A}\, \boldsymbol{b}^{(\ell)}$
  \State $\mathbf{A} \gets \mathbf{A}\, \mathbf{W}^{(\ell)}$
  \State \Return $(\mathbf{A}, \boldsymbol{d})$
\EndProcedure
\Statex
\Procedure{Concretize}{$\lA, \ld, \uA, \ud, \boldsymbol{x}_0, \varepsilon$} 
  \State $\uA_+ \gets \max(\uA, 0)$, \quad $\uA_- \gets \min(\uA, 0)$
  \State $\lA_+ \gets \max(\lA, 0)$, \quad $\lA_- \gets \min(\lA, 0)$
  \State $\boldsymbol{U} \gets \uA_+ (\boldsymbol{x}_0 + \varepsilon \onevec) + \uA_- (\boldsymbol{x}_0 - \varepsilon \onevec) + \ud$
  \State $\boldsymbol{L} \gets \lA_+ (\boldsymbol{x}_0 - \varepsilon \onevec) + \lA_- (\boldsymbol{x}_0 + \varepsilon \onevec) + \ld$
  \State \Return $(\boldsymbol{L}, \boldsymbol{U})$
\EndProcedure
\end{algorithmic}
\end{algorithm}

\parhead{Time complexity} For a neural network with $m$ layers, $n$ neurons per layer, and $n$ output neurons, \base runs in time $O(m^2n^3).$ The algorithm performs two nested for-loops, consisting of the outer forward pass and the inner backward pass, each over $m$ layers. At each layer, \base performs multiplications of $n\times n$ matrices, each requiring time $O(n^3)$.

\subsection{The Relaxation Algorithm}
\label{app:crown-details:relaxation}

\begin{figure}[h]
  \centering
  \begin{minipage}[t]{2.3in}
    \centering
    \includegraphics[width=\linewidth]{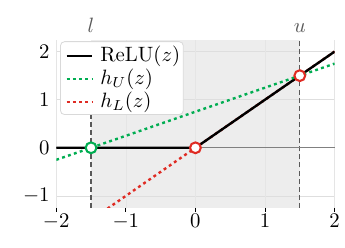}
  \end{minipage}
  \hspace{0.15in}
  \begin{minipage}[t]{2.3in}
    \centering
    \includegraphics[width=\linewidth]{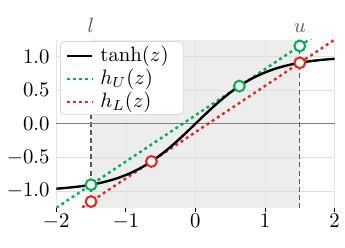}
  \end{minipage}
  \caption{Linear relaxations of the ReLU (left) and $\tanh$ (right) activation functions over the interval $[l, u]$ (shaded). Each activation $\sigma(z)$ (black) is bounded by affine functions $h_U(z)$ (dashed green) and $h_L(z)$ (dashed red) over the interval $[l, u]$.}
  \label{fig:linear-relaxations}
\end{figure}

The \textsc{Relax} algorithm takes as input a concretized interval $(l,u)$ and an activation function $\sigma$, and outputs scalars $(\ua,\ubof,\la,\lbof)$ satisfying
\begin{equation}
  \la \preact + \lbof \;\le\; \sigma(\preact) \;\le\; \ua \preact + \ubof, \qquad \forall \preact \in [l,u].
\end{equation}
This algorithm is called extensively by \base to produce per-neuron linear relaxations which can then be verified by \tool via the Four-Point Relaxation Gadget in Section~\ref{sec:gadget}.

Below, we describe the \textsc{Relax} procedure for S-shaped activations $\sigma$. The procedure for ReLU activations is provided in Section~\ref{sec:gadget:relu}. We first describe the selection of the upper bound $h_U(\preact) = \ua \preact + \ubof$; the lower bound follows by reflection.

\parhead{Procedure}
If $l=u$, output the tangent $h_U(z)=\sigma'(l)(z-l)+\sigma(l)$ and the same line for $h_L$. For $l<u$, \base selects a valid pair $(\ua,\ubof)$ via the following steps:
\begin{enumerate}
  \item If $l=0$, set $\tau=0$ by convention and proceed to Case~2. Otherwise, solve the equation
  \begin{equation}
    \label{eq:crown-relax-selection}
  \sigma'(\tau) = \frac{\sigma(\tau) - \sigma(l)}{\tau - l}
  \end{equation}
  for $\tau\ne l$. Intuitively, this equation asks for a value $\tau$ such that the tangent line of $\sigma(z)$ at $z=\tau$ is exactly equal to the secant line through $(l,\sigma(l))$ and $(\tau,\sigma(\tau))$. Under Assumption~\ref{ass:sshape}, this equation has a unique solution for $l\ne0$, with $\tau>0$ if $l<0$ and $\tau<0$ if $l>0$ (Lemma~\ref{lemma:tangent-secant-root} and reflection).
  \item Choose an upper bound $h_U(z)=\ua z+\ubof$ based on where $\tau$ falls relative to the interval $[l, u]$:
    \begin{itemize}
      \item \textbf{Case 1: $\tau \in (l, u)$.} Output the tangent line for $\sigma$ at $z=\tau$:
      \[
        h_U(\preact) = \frac{\sigma(\tau) - \sigma(l)}{\tau - l}(\preact - l) + \sigma(l).
      \]
      \item \textbf{Case 2: $\tau \le l$.} Then $l\ge0$ and $\sigma$ is concave on $[l, u]$, so any tangent line at a point in this interval is an upper bound for $\sigma$ over the interval. Output the tangent at the midpoint $\mu:=\frac{l + u}{2}$:
      \[
      h_U(\preact)=\sigma'(\mu)(\preact-\mu)+\sigma(\mu).
      \]
      \item \textbf{Case 3: $\tau \ge u$.} Output the secant line through the two interval endpoints:
      \[
        h_U(\preact) = \frac{\sigma(u) - \sigma(l)}{u - l}(\preact - l) + \sigma(l).
      \]
    \end{itemize}
  \item Output $(\ua, \ubof)$ from the slope and intercept of $h_U$.
\end{enumerate}
For the lower bound, apply the same procedure to the reflected activation
$\widetilde\sigma(x)=-\sigma(-x)$ on $[-u,-l]$. If $\widetilde h_U$ is
the resulting upper line, set $h_L(z)=-\widetilde h_U(-z)$.
The reflected activation satisfies the same assumption. The proof that
this selection satisfies the four-point criteria is given in
Appendix~\ref{proof:fourpoint}.

\begin{lemma}[Tangent through a fixed left endpoint]
\label{lemma:tangent-secant-root}
Let $\sigma$ satisfy Assumption~\ref{ass:sshape} and fix $l<0$.
There is a unique positive point $\tau_l$ at which the tangent to
$\sigma$ passes through $(l,\sigma(l))$. Equivalently,
\[
  \sigma'(\tau_l)=\frac{\sigma(\tau_l)-\sigma(l)}{\tau_l-l}.
\]
Moreover, let $q_l(t)=(\sigma(t)-\sigma(l))/(t-l)$ be the slope of
the secant through $(l,\sigma(l))$ and $(t,\sigma(t))$, for $t>l$.
This slope strictly increases until $t=\tau_l$ and strictly decreases
afterward. Thus, the secant with the largest slope is also tangent
to $\sigma$ at its right endpoint $\tau_l$.
\end{lemma}
\begin{proof}
For $t\ge0$, define
\[
 T_l(t)=\sigma(t)+(l-t)\sigma'(t).
\]
This is the height at $x=l$ of the tangent line drawn at $t$.
We seek a point where this height equals $\sigma(l)$.

As $t$ moves right along the concave branch, $T_l(t)$ strictly
increases. Indeed, for $0\le s<t$, concavity gives
$\sigma(t)-\sigma(s)\ge(t-s)\sigma'(t)$, so
\[
 T_l(t)-T_l(s)
 \ge (s-l)\bigl(\sigma'(s)-\sigma'(t)\bigr)>0,
\]
where the last inequality uses $s-l>0$ and the strict decrease
of $\sigma'$ on $[0,\infty)$.

At $t=0$, strict convexity on $[l,0]$ places the tangent below
the fixed endpoint: $T_l(0)=\sigma(0)+lB<\sigma(l)$.
For large $t$, the tangent's height is above that endpoint.
To see this, let $\sigma_+=\lim_{t\to\infty}\sigma(t)$, which exists
by boundedness and monotonicity. Since $\sigma'$ is nonnegative
and decreasing on the positive half-line,
\[
 0\le t\sigma'(t)
 \le 2\int_{t/2}^t\sigma'(v)\,dv
 =2\bigl(\sigma(t)-\sigma(t/2)\bigr)\longrightarrow0.
\]
Consequently $T_l(t)\to\sigma_+>\sigma(l)$.
Continuity and strict increase of $T_l$ therefore give exactly
one point $\tau_l>0$ with $T_l(\tau_l)=\sigma(l)$.

Finally, the secant slope satisfies
\[
 q_l'(t)=\frac{\sigma'(t)-q_l(t)}{t-l}\qquad(t>l).
\]
For $l<t\le0$, strict convexity gives $\sigma'(t)>q_l(t)$,
so $q_l'(t)>0$. For $t\ge0$, the same derivative can be written as
\[
 q_l'(t)=\frac{\sigma(l)-T_l(t)}{(t-l)^2}.
\]
It is positive before $\tau_l$, zero at $\tau_l$, and negative
afterward. This proves the claimed increase and decrease of $q_l$.
\end{proof}

\section{Details of \tool}
\label{app:panda-details}

In this appendix, we present the \tool system in full detail. Algorithm~\ref{alg:gadget} presents the prover for the \textit{Four-Point Relaxation Gadget}, accompanying Section~\ref{sec:gadget:sshape}. Appendix~\ref{app:fd} further shows how to eliminate the $T_{\sigma'}$ table in Algorithm~\ref{alg:gadget} through an optimization. Algorithms~\ref{alg:panda} and~\ref{alg:verify} present the full \tool prover and verifier procedures, respectively.

In our algorithms, $\Pilookup$ uses three precomputed public lookup tables over the quantization domain $\mathcal{D}$:
\[
  T_\mathrm{ReLU} := \{(x, \mathrm{ReLU}(x)) : x \in \mathcal{D}\}, \quad
  T_\sigma := \{(x, \sigma(x)) : x \in \mathcal{D}\}, \quad
  T_{\sigma'} := \{(x, \sigma'(x)) : x \in \mathcal{D}\}.
\]
$\Pilookup$ is also used to perform range proofs, as discussed in Section~\ref{sec:pre}.

\begin{algorithm}[H]
\caption{Four-Point Relaxation Gadget prover (upper bound).}
\label{alg:gadget}
\begin{algorithmic}[1]
\Require Witness $(l, u, \ua, \ubof, z_1, z_2)$ and lookup tables $T_\sigma$, $T_{\sigma'}$
\State Compute $s_l \gets \sigma(l),\ s_u \gets \sigma(u),\ s_1 \gets \sigma(z_1),\ s_2 \gets \sigma(z_2),\ s'_1 \gets \sigma'(z_1),\ s'_2 \gets \sigma'(z_2)$.
\State $\Pilookup.\textsc{Prove}\left((l,s_l),(u,s_u),(z_1,s_1),(z_2,s_2)\in T_\sigma\right).$
\State $\Pilookup.\textsc{Prove}\left((z_1,s_1'),(z_2,s_2')\in T_{\sigma'}\right)$.
\State $\Pilookup.\textsc{Prove}(\ua l + \ubof - s_l \ge 0).$\Comment{Left endpoint}
\State $\Pilookup.\textsc{Prove}(\ua u + \ubof - s_u \ge 0)$.\Comment{Right endpoint}
\State $\Pilookup.\textsc{Prove}(z_1 \le 0)$ and $\Pilookup.\textsc{Prove}(z_2 \ge 0)$.\Comment{Tangency point ordering}
\State $\Piarith.\textsc{Prove}(s_i'=\ua)$ for $i \in \{1, 2\}$.\Comment{Tangency point correctness}
\State $\Pilookup.\textsc{Prove}(\ua z_1 + \ubof - s_1 \ge 0)$.\Comment{Left tangency condition}
\State $\Pilookup.\textsc{Prove}\left((u - z_2)(\ua z_2 + \ubof - s_2) \ge 0\right)$.\Comment{Right tangency condition}
\State \Return all arithmetic, lookup, and range proofs generated.
\end{algorithmic}
\end{algorithm}

\begingroup
\setlength{\intextsep}{4pt}
\setlength{\parskip}{0pt}
\begin{algorithm}[H]
\caption{The \tool prover.}
\label{alg:panda}
\setlength{\baselineskip}{10pt}
\begin{algorithmic}[1]
\Statex \textbf{Public statement.} Input point $\boldsymbol{x}_0\in\R^{n_0}$, perturbation budget $\varepsilon>0$, model architecture, robustness claim $(\mathbf{C},\boldsymbol{u})$, fixed model commitment, and commitments to witnesses. The activation funtions and lookup tables are public.
\vspace{0.1em}
\Statex \textbf{Private witnesses.} Model parameters $\{\mathbf{W}^{(k)},\boldsymbol{b}^{(k)}\}_{k=1}^m$ and all values in both bound directions of the quantized \base transcript (Section~\ref{sec:panda}). $\prover$ commits to the witnesses via $\Picom$.
\vspace{0.1em}
\Statex \textbf{Prover claim.} $\forall\boldsymbol{x}\in\Ball(\boldsymbol{x}_0,\varepsilon),\ \mathbf{C}f(\boldsymbol{x})\le\boldsymbol{u}$.
\vspace{0.1em}
\Statex \textbf{Proof.} $\prover$ runs Algorithm~\ref{alg:crown}, appending every proof below to the prescribed indexed list $\Pi$.
\vspace{0.1em}

\Statex \textit{(0) Initialization and connections.} For every target layer and both bound directions:
\Statex Ensure $(\mathbf A,\boldsymbol d)=(\mathbf W^{(k)},\boldsymbol b^{(k)})$ for $k<m$, or $(\mathbf C\mathbf W^{(m)},\mathbf C\boldsymbol b^{(m)})$.
\Statex All repeated values share commitments or equality proofs.

\vspace{0.1em}

\Statex \textit{(1) Signed components.} For each call $\mathbf{A}_+ \gets \max(\mathbf{A}, 0),\; \mathbf{A}_- \gets \min(\mathbf{A}, 0)$:
\State $\pi_+\gets\Pilookup.\textsc{Prove}\bigl((\mathbf A,\mathbf A_+)\subset T_\mathrm{ReLU}\bigr)$
\State $\pi_-\gets\Pilookup.\textsc{Prove}\bigl((-\mathbf A,-\mathbf A_-)\subset T_\mathrm{ReLU}\bigr)$
\State $\Pi \gets \Pi\cup\{\pi_+,\pi_-\}$

\vspace{0.1em}

\Statex \textit{(2) Activation-relaxation update.} For each \textsc{Step1} (upper pass shown):
\State $\pi_1\gets\Piarith.\textsc{Prove}\bigl(\mathbf A'=\mathbf A_+\operatorname{diag}(\bua^{(\ell)})+\mathbf A_-\operatorname{diag}(\bla^{(\ell)})\bigr)$
\State $\pi_2\gets\Piarith.\textsc{Prove}\bigl(\boldsymbol d'=\boldsymbol d+\mathbf A_+\bubof^{(\ell)}+\mathbf A_-\blbof^{(\ell)}\bigr)$
\State $\Pi\gets \Pi\cup \{\pi_1,\pi_2\}$
\Statex For the lower pass, exchange $(\bua,\bubof)$ with $(\bla,\blbof)$ in both equations.

\vspace{0.1em}

\Statex \textit{(3) Linear-map substitution.} For each \textsc{Step2}:
\State $\pi_3\gets\Piarith.\textsc{Prove}(\mathbf A'=\mathbf A\mathbf W^{(\ell)})$
\State $\pi_4\gets\Piarith.\textsc{Prove}(\boldsymbol d'=\boldsymbol d+\mathbf A\boldsymbol b^{(\ell)})$
\State $\Pi\gets \Pi\cup \{\pi_3,\pi_4\}$

\vspace{0.1em}

\Statex \textit{(4) Concretization.} For each \textsc{Concretize}:
\State Prove signed components for $\uA,\lA$ by step~(1).
\State $\pi_5\gets\Piarith.\textsc{Prove}\bigl(\boldsymbol U=\uA_+(\boldsymbol x_0+\varepsilon\onevec)+\uA_-(\boldsymbol x_0-\varepsilon\onevec)+\ud\bigr)$
\State $\pi_6\gets\Piarith.\textsc{Prove}\bigl(\boldsymbol L=\lA_+(\boldsymbol x_0-\varepsilon\onevec)+\lA_-(\boldsymbol x_0+\varepsilon\onevec)+\ld\bigr)$
\State $\Pi\gets \Pi\cup \{\pi_5,\pi_6\}$

\vspace{0.1em}

\Statex \textit{(5) Activation relaxation.} For each hidden neuron with its committed bounds:
\State $\pi_u\gets\textsc{Relaxation-Gadget}(\sigma,\boldsymbol L,\boldsymbol U,\ua,\ubof)$  \Comment{calls upper bound Alg.~\ref{alg:gadget}}
\State $\pi_l\gets\textsc{Relaxation-Gadget}(\sigma,\boldsymbol L,\boldsymbol U,\la,\lbof)$  \Comment{lower bound is analogous}
\State $\Pi\gets \Pi\cup \{\pi_u,\pi_l\}$

\vspace{0.1em}

\Statex \textit{(6) Final robustness claim.}
After Algorithm~\ref{alg:crown} terminates at layer $m$, $\prover$ proves that the final concretized upper bound $\boldsymbol{U}^{(m)}$ is at most $\boldsymbol{u}$ entry-wise via a range proof.
\State $\pi_\mathrm{final}\gets \Pilookup.\textsc{Prove}(\boldsymbol{u}-\boldsymbol{U}^{(m)}\geq 0)$
\State $\Pi\gets \Pi\cup\{\pi_\mathrm{final}\}$

\vspace{0.1em}

\Statex \textit{(7) Quantization rescaling.} For each rescaling $q_z=\left\lfloor S_z q_x/S_x\right\rfloor$ with committed remainder $r$:
\State $\pi_{\mathrm{rem}}\gets\Pilookup.\textsc{Prove}(0\le r<S_x)$
\Comment{This is rewritten as $r\ge 0\wedge S_x-r>0$.}
\State $\pi_{\mathrm{div}}\gets\Piarith.\textsc{Prove}(S_zq_x=S_xq_z+r)$
\State $\Pi\gets \Pi\cup \{\pi_{\mathrm{rem}},\pi_{\mathrm{div}}\}$
\vspace{0.1em}

\State \Return $\Pi$.
\end{algorithmic}
\end{algorithm}

\begin{algorithm}[H]
\caption{The \tool verifier.}
\label{alg:verify}
\setlength{\baselineskip}{10pt}
\begin{algorithmic}[1]
\Require Public statement $x$ and proof $\Pi$ from Algorithm~\ref{alg:panda}.
\State Derive the required proof list and local statements from the public parameters and shared commitments; reject missing, duplicate, extra, malformed, or misindexed items.
\For{each proof $\pi\leftarrow \Piarith.\textsc{Prove}$ in $\Pi$}
    \State run $\Piarith.\textsc{Verify}(x,\pi)$
\EndFor
\For{each proof $\pi\leftarrow\Pilookup.\textsc{Prove}$ in $\Pi$}
  \State run $\Pilookup.\textsc{Verify}(x,\pi)$
\EndFor
\For{each relaxation-gadget proof $\pi\in \Pi$}
    \State run $\textsc{Verify-Relaxation-Gadget}(x,\pi)$
\EndFor
\State \textbf{accept} iff every check above passes
\end{algorithmic}
\end{algorithm}
\endgroup

\newpage

\section{Comparison with FairProof}\label{comparison}
FairProof~\citep{fairproof} is the only suitable direct comparison, since it is the only ZKP system that certifies local robustness, to the best of our knowledge. We evaluate FairProof's own system head-to-head with \tool on shared benchmarks: the FairProof Adult model and our smallest MNIST models on the first 5 properties \tool proves. FairProof's offline phase was never released, so we reimplemented it and validated it against their released example, using a single core and the same CPUs as in \tool's evaluation. We present the results in Table~\ref{fairproof_compare}.

Where FairProof runs (only on FairProof Adult), \tool requires about half (34.89s) of FairProof’s prover time (66.0s in total). FairProof's verifier time is smaller, but \tool's is still on the order of seconds. However, FairProof cannot scale to any of the larger MNIST models.
\begin{table}[H]
  \caption{FairProof's performance on small benchmarks.}
  \label{fairproof_compare}
  \centering
  \begin{threeparttable}
    \begin{tabular}{lrrrr}
      \toprule
      Model & Offline prove (s) & Online prove (s) & Verify (s) & Proof (KB) \\
      \midrule
      FairProof Adult
        & $9.2$ & $56.8$ & $0.02$ & $19.1$ \\
      MNIST $2\times[20]$ ReLU
        & $1908.0^{\dagger}$
        & \makecell[r]{4/5 timeout ($>$1h),\\ 1/5 OOM ($>$512GiB)} & --- & --- \\
      MNIST $3\times[20]$ ReLU
        & timeout ($>$12h) & --- & --- & --- \\
      MNIST $2\times[1024]$ ReLU
        & timeout ($>$12h) & --- & --- & --- \\
      \bottomrule
    \end{tabular}
    \begin{tablenotes}[flushleft]
    \small
      \item[$\dagger$] The offline phase completed, but the online phase was killed.
    \end{tablenotes}
  \end{threeparttable}
\end{table}

\section{Complete Evaluation}\label{complete_evaluation}

\begin{table}[H]
  \caption{Evaluation. Each row is one model. (adv.) denotes adversarially trained models. 100 tests are run for each model, except for FairProof (only one test is run) and MNIST (only images correctly classified by the model are tested). Prove/verify times show $\mu\pm \sigma$ (mean and standard deviation) over the verified subset only. N is the number of benchmarks evaluated for each model, which is also the number of properties that CROWN certifies for that model.}
  \label{complete_final_least_table}
  \centering
  \small
  \setlength{\tabcolsep}{3.2pt}
  \begin{tabular}{@{}lllS[separate-uncertainty,table-format=3.2(3.1)]S[separate-uncertainty,table-format=2.2(2.1)]rrrr@{}}
    \toprule
    Dataset & Structure & Act. & {Prove (s)} & {Verify (s)} & {Proof} & {CROWN} & {Drift} & {N}\\
            &           &      &             &              & {(MB)} & {(s)} & {(\%)} &       \\
    \midrule
    \multicolumn{9}{@{}l}{\textit{MNIST}} \\
     & $2 \times [20]$ & ReLU & 22.81(20) & 0.97(1) & 3 & 0.00013 & 0.02 & 97\\
     & $2 \times [20]$ & Sigmoid & 58.35(81) & 2.33(3) & 8 & 0.00014 & 0.19 & 94\\
     & $2 \times [20]$ & Tanh & 58.77(39) & 2.34(3) & 8 & 0.00014 & 0.16 & 97\\
     & $3 \times [20]$ & ReLU & 42.64(51) & 1.90(3) & 6 & 0.00018 & 0.04 & 95\\
     & $3 \times [20]$ & Sigmoid & 99.62(120) & 4.21(5) & 14 & 0.00022 & 0.55 & 93\\
     & $3 \times [20]$ & Tanh & 105.56(233) & 4.46(10) & 15 & 0.00019 & 0.40 & 93\\
     & $2 \times [1024]$ & ReLU & 41.70(72) & 1.69(2) & 4 & 0.06 & 0.02 & 98\\
     & $2 \times [1024]$ & Sigmoid & 74.60(149) & 3.26(6) & 9 & 0.06 & 0.62 & 99\\
     & $2 \times [1024]$ & Tanh & 83.17(338) & 3.63(15) & 9 & 0.07 & 0.39 & 98\\
     & $3 \times [1024]$ (adv.) & ReLU & 90.79(44) & 4.68(3) & 9 & 0.22 & -0.18 & 97\\
     & $3 \times [1024]$ & ReLU & 101.58(375) & 4.60(13) & 9 & 0.22 & 0.03 & 98\\
     & $3 \times [1024]$ & Sigmoid & 172.29(304) & 7.76(14) & 19 & 0.21 & 1.66 & 98\\
     & $3 \times [1024]$ & Tanh & 179.27(326) & 7.83(15) & 20 & 0.21 & 1.83 & 98\\
     & $4 \times [1024]$ (adv.) & ReLU & 172.16(389) & 8.96(15) & 16 & 0.45 & 0.01 & 98\\
     & $4 \times [1024]$ & ReLU & 226.08(1467) & 9.56(48) & 17 & 0.44 & -0.01 & 100\\
     & $4 \times [1024]$ & Sigmoid & 327.47(1136) & 14.13(46) & 31 & 0.47 & 2.48 & 97\\
     & $4 \times [1024]$ & Tanh & 333.94(581) & 14.31(28) & 32 & 0.46 & 6.33 & 99\\
    \addlinespace
    \multicolumn{9}{@{}l}{\textit{Other}} \\
    SafeNLP\,med & $2 \times [128]$ & ReLU & 23.04(19) & 0.81(1) & 3 & 0.00013 & -0.05 & 54\\
    SafeNLP\,robot & $2 \times [128]$ & ReLU & 20.33(19) & 0.78(1) & 3 & 0.00014 & -0.29 & 3\\
    LunarLander & $3 \times [64]$ & ReLU & 36.61(29) & 1.46(2) & 5 & 8.1e-5 & -0.09 & 18\\
    FairProof\,Adult & $[8, 2]$ & ReLU & 34.89 & 1.21 & 4 & 1.4e-5 & -12.60 & 1\\
    \bottomrule
  \end{tabular}
\end{table}

\section{Prover Breakdown}\label{prover_breakdown}

\begin{table}[H]
  \caption{Component-wise breakdown of the prover time.
    \texttt{commit} = polynomial commitments to every witness (Hyrax~\cite{hyrax});
    \texttt{matmul} = the matrix-multiplication verification (\cite{thaler13});
    \texttt{lookup} = all lookup arguments (LogUp-GKR~\cite{logup-GKR});
    \texttt{other} = other cryptographic computation;
    total in Cryptographic = the sum of the four preceding columns;
    total in Non-crypto = total non-cryptographic computation time of the prover
    (including the plaintext quantized-CROWN that generates witnesses);
    \texttt{tangent} = time for computing the tangent points used in linear
    relaxations of S-shaped activation functions, which is a part of
    total in . Upper panel: absolute times; lower panel: shares of
    \texttt{prove}. Note the unit change on \texttt{tangent}.}
  \label{tab:prover-breakdown}
  \centering
  \small
  \begin{tabular*}{\textwidth}{@{\extracolsep{\fill}}ll r rrrrr rr@{}}
    \toprule
    & & &
    \multicolumn{5}{c}{Cryptographic} &
    \multicolumn{2}{c}{Non-crypto} \\
    \cmidrule(lr){4-8}\cmidrule(l){9-10}
    Structure & Act. & \texttt{prove}
      & \texttt{commit} & \texttt{matmul} & \texttt{lookup} & \texttt{other} & total
      & \texttt{tangent} & total \\
    \midrule
    \multicolumn{10}{@{}l}{\emph{Absolute time (s); \texttt{tangent} in ms}} \\
    $3\times[20]$   & ReLU    &  45.08 &  0.93 &  0.35 &  43.59 & 0.21 &  45.07 & 0.000 &  0.01 \\
    $3\times[20]$   & Sigmoid &  92.34 &  0.86 &  0.33 &  90.25 & 0.89 &  92.33 & 0.009 &  0.01 \\
    $3\times[20]$   & Tanh    & 116.45 &  0.94 &  0.36 & 113.27 & 1.87 & 116.44 & 0.013 &  0.01 \\
    $3\times[1024]$ & ReLU    & 106.95 & 15.70 &  8.54 &  66.08 & 0.51 &  90.82 & 0.000 & 16.12 \\
    $3\times[1024]$ & Sigmoid & 170.13 & 16.17 &  8.50 & 127.50 & 2.70 & 154.86 & 0.022 & 15.26 \\
    $3\times[1024]$ & Tanh    & 196.80 & 20.62 & 10.49 & 144.39 & 3.42 & 178.91 & 0.646 & 17.84 \\
    \addlinespace[2pt]
    \midrule
    \multicolumn{10}{@{}l}{\emph{As \% of \texttt{prove} (\texttt{prove} in s)}} \\
    $3\times[20]$   & ReLU    &  45.08 &  2.1 & 0.8 & 96.7 & 0.5 & 100.0 & $<$0.01 &  0.0 \\
    $3\times[20]$   & Sigmoid &  92.34 &  0.9 & 0.4 & 97.7 & 1.0 & 100.0 & $<$0.01 &  0.0 \\
    $3\times[20]$   & Tanh    & 116.45 &  0.8 & 0.3 & 97.3 & 1.6 & 100.0 & $<$0.01 &  0.0 \\
    $3\times[1024]$ & ReLU    & 106.95 & 14.7 & 8.0 & 61.8 & 0.5 &  84.9 & $<$0.01 & 15.1 \\
    $3\times[1024]$ & Sigmoid & 170.13 &  9.5 & 5.0 & 74.9 & 1.6 &  91.0 & $<$0.01 &  9.0 \\
    $3\times[1024]$ & Tanh    & 196.80 & 10.5 & 5.3 & 73.4 & 1.7 &  90.9 & $<$0.01 &  9.1 \\
    \bottomrule
  \end{tabular*}
\end{table}

We provide the component-wise breakdown of the prover time for three
MNIST $3 \times [20]$ models and three MNIST $3 \times [1024]$ models in Table~\ref{tab:prover-breakdown}.
Every row shows the mean over the first $5$ properties \tool proves.

We identify the largest component as the lookup arguments, but we still observe
that all components scale polynomially in the size of the network, which aligns
with the theoretical guarantees of these cryptographic protocols.

\section{Quantization Error}
\label{app:finite-encodings}
In this appendix, we precisely describe the guarantees provided by \tool and the effects of quantization errors. The correctness of the \base local robustness protocol assumes real-number arithmetic. The \tool algorithm quantizes the computation as well as the verification and performs a ZKP over finite fields. The ZKP of \tool proves satisfaction of the encoded quantized certification constraints. Interpreting acceptance as robustness of the quantized network additionally requires numerical soundness of those constraints, which we have not established. Also, \tool cannot provably guarantee the outcome of running \base over real numbers or floating-point; the process of quantization induces small errors relative to real number computation, initially due to rounding, and then propagating across computation. 

In our analysis, we empirically assess the impact of quantization error. In the \textit{Drift} column of Table~\ref{final_least_table}, we investigate the magnitude of the difference between \base and \tool output bounds and show this is limited to single-digit percentages, which can be reduced even further by increasing the quantization precision. Further, we show that \tool certifies all claims that \base certifies under our prescribed precision. A further step is to provably bound the quantization drift by a certain amount to provide a rigorous guarantee on the impact of quantization. This is an inherently difficult task, particularly in light of the fact that floating-point computation may contain errors. However, the approach of interval analysis presented in Azkaban~\cite{azkaban} provides a potential framework for adapting \tool to achieve provable guarantees of quantization error.

\section{Optimization: Finite Differences}
\label{app:fd}
The Four-Point Relaxation Gadget in Algorithm~\ref{alg:gadget} uses lookup
proofs against both $T_\sigma$ and $T_{\sigma'}$. Precomputing both tables
can be a memory bottleneck. Finite differences can replace derivative
lookups when the local conditions below hold.

For a derivative root $z$ satisfying $\sigma'(z)=\alpha$, we seek a nearby
point $z_0$ whose two secant slopes bracket $\alpha$:
\begin{equation}\label{eq:fd}
 \frac{\sigma(z_0)-\sigma(z_0-\delta)}{\delta}\le\alpha\le
 \frac{\sigma(z_0+\delta)-\sigma(z_0)}{\delta}.
\end{equation}
The inequalities are reversed on a concave branch. These checks use only
$T_\sigma$.

\begin{theorem}[Finite differences]\label{fd}
Let $\sigma\in C^1$ on $[z-2\delta,z+2\delta]$, where $\delta>0$,
and let $\sigma'(z)=\alpha$. Suppose $\sigma'$ is strictly increasing
on this entire window. Let a finite grid $\mathcal D$ contain all points
of a lattice of spacing $\eta\le\delta$ in this window, with
$\delta$ an integer multiple of $\eta$. Then some
$z_0\in\mathcal D\cap(z-\delta,z+\delta)$ satisfies~\eqref{eq:fd},
and both shifted points $z_0\pm\delta$ belong to $\mathcal D$.
If $\sigma'$ is strictly decreasing on the window, the same conclusion
holds with reversed inequalities.
\end{theorem}
\begin{proof}
Define $f(x)=\sigma(x)-\sigma(x-\delta)$ for
$x\in[z-\delta,z+2\delta]$. In the increasing case,
$f'(x)=\sigma'(x)-\sigma'(x-\delta)>0$ in the interior,
so $f$ is strictly increasing on its domain. Integrating the strictly
increasing derivative gives
\[
 f(z)<\delta\alpha<f(z+\delta).
\]
There is a unique $z_+\in(z,z+\delta)$ with $f(z_+)=\delta\alpha$.
Put $z_-=z_+-\delta$. For every $z_0\in[z_-,z_+]$, monotonicity gives
$f(z_0)\le\delta\alpha\le f(z_0+\delta)$, which is~\eqref{eq:fd}.
The interval $[z_-,z_+]$ has length $\delta\ge\eta$, lies inside
$(z-\delta,z+\delta)$, and therefore contains a lattice point.
Its shifts by $\delta$ are on the same lattice and inside the larger
window, hence belong to $\mathcal D$. Strict decrease reverses this
argument.
\end{proof}

\begin{lemma}[Finite-difference converse]\label{lemma:fd-converse}
If $\sigma$ is $C^1$ on $[z_0-\delta,z_0+\delta]$ and $\delta>0$, then
\[
 [\sigma(z_0)-\sigma(z_0-\delta)-\delta\alpha]
 [\sigma(z_0+\delta)-\sigma(z_0)-\delta\alpha]\le0
\]
implies a root $z\in(z_0-\delta,z_0+\delta)$ of $\sigma'(z)=\alpha$.
\end{lemma}
\begin{proof}
The mean value theorem represents the two secant slopes as
$\sigma'(z_0^-)$ and $\sigma'(z_0^+)$, with
$z_0^-\in(z_0-\delta,z_0)$ and $z_0^+\in(z_0,z_0+\delta)$.
The product condition says these values bracket $\alpha$, including
equality. Continuity of $\sigma'$ gives a root in $[z_0^-,z_0^+]$,
which lies strictly inside the stated interval.
\end{proof}

\parhead{Applicability to Algorithm~\ref{alg:gadget}}
For $0<\alpha<B$, a stencil contained in the negative or positive branch
locates the corresponding unique derivative root from
Lemma~\ref{lemma:two_tangency_points}. The point $z_0$ is a
$\delta$-approximation to that root; it is not necessarily itself a
point of tangency. The four-point conditions remain conditions at the
actual derivative roots.

The existence guarantee requires the enlarged window
$[z-2\delta,z+2\delta]$ to remain within one monotone branch of $\sigma'$.
When it crosses the inflection point, the guarantee does not apply;
one must use a smaller admissible stencil or retain the derivative
lookup. Simply substituting $z_0=0$ does not certify a root for
$0<\alpha<B$. For $\alpha=B$, the unique derivative root is exactly
$z=0$, and the gadget uses the coinciding points $z_1=z_2=0$.

\section{Proof of Lemma~\ref{lemma:two_tangency_points}}\label{proof:two_tangency_points}
\begin{proof}
By Assumption~\ref{ass:sshape}, $\sigma$ is strictly convex on $(-\infty,0]$ and strictly concave on $[0,+\infty)$. Its derivative $\sigma'$ is continuous, bounded below by 0, strictly increasing on $(-\infty,0]$ and strictly decreasing on $[0,+\infty)$. In particular, $\sigma'(x)>0$ at every finite point, and $\sigma'$ attains its maximum at $x=0$, where $\sigma'(0)=B$. Additionally, both tail limits exist by monotonicity; a positive limit would contradict boundedness of $\sigma$ after integration. Hence $\lim_{x\rightarrow \pm \infty} \sigma'(x)=0$. Therefore, by the Intermediate Value Theorem and strict monotonicity of $\sigma'$, every point $\alpha \in (0,B)$ has exactly one $\sigma'$-preimage in $(-\infty,0)$ and exactly one $\sigma'$-preimage in $(0,+\infty)$. For $\alpha = B$, the unique preimage is $z_1 = z_2 = 0$, since $\sigma'$ attains its maximum $B$ only at $x=0$.
\end{proof}

\section{Proof of Theorem \ref{thm:casefreefourpoint}}
\label{proof:fourpoint}

\begin{figure}[htbp]
  \centering
  \begin{minipage}[t]{2.3in}
    \centering
    \includegraphics[width=\linewidth]{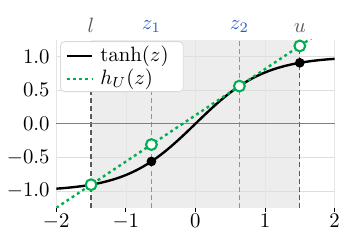}
  \end{minipage}
  \hspace{0.15in}
  \begin{minipage}[t]{2.3in}
    \centering
    \includegraphics[width=\linewidth]{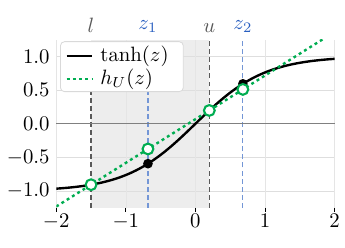}
  \end{minipage}
  \caption{Case-free four-point soundness (Theorem~\ref{thm:casefreefourpoint}) for the activation function $\sigma(z)=\tanh(z).$  The left image depicts $[l,u]=[-1.5,1.5]$, $z_2\le u$, and $\sigma(z_2)\le h_U(z_2)$. The right image depicts $[l,u]=[-1.5,0.2]$, $z_2> u$, and $\sigma(z_2)>h_U(z_2).$ Both cases therefore satisfy the criteria given by Theorem~\ref{thm:casefreefourpoint} without requiring case-specific analysis.}
  \label{fig:four-point-soundness}
\end{figure}

\textbf{Theorem \ref{thm:casefreefourpoint}, part (a) (Case-free four-point soundness).}
\textit{Let $\sigma$ satisfy Assumption~\ref{ass:sshape}, and fix an interval $[l, u] \subseteq \mathbb{R}$. Consider a candidate linear \textbf{upper bound} $h_U(\preact) = \ua \preact + \ubof$ with $\ua \in (0, B]$ and two (possibly coinciding) points of tangency $z_1 \le 0 \le z_2$ satisfying $\sigma'(z_1) = \sigma'(z_2) = \ua$. If the following conditions are satisfied, then $h_U$ is a valid upper bound on $[l, u]$.
\begin{itemize}
    \item[(i)] Endpoints: $\sigma(l) \le h_U(l)$ and $\sigma(u) \le h_U(u)$.
    \item[(ii)] Left point of tangency: $\sigma(z_1) \le h_U(z_1)$.
    \item[(iii)] Right point of tangency: $0 \le (u - z_2)\bigl(h_U(z_2) - \sigma(z_2)\bigr)$.
\end{itemize}
}

\textit{Consider a candidate linear \textbf{lower bound} $h_L(\preact) = \la \preact + \lbof$ with $\la \in (0, B]$ and two (possibly coinciding) points of tangency $z_1 \le 0 \le z_2$ satisfying $\sigma'(z_1) = \sigma'(z_2) = \la$. If the following conditions are satisfied, then $h_L$ is a valid lower bound on $[l, u]$.
\begin{itemize}
    \item[(i)] Endpoints: $\sigma(l) \ge h_L(l)$ and $\sigma(u) \ge h_L(u)$.
    \item[(ii)] Right point of tangency: $\sigma(z_2) \ge h_L(z_2)$.
    \item[(iii)] Left point of tangency: $0 \le (z_1 - l)\bigl(\sigma(z_1) - h_L(z_1)\bigr)$.
\end{itemize}
}
\begin{proof}
    We prove soundness for the upper bound, and the lower bound is analogous. Define 
    \[
    f(\preact) := h_U(\preact) - \sigma(\preact).
    \]
    It suffices to show $f \ge 0$ on $[l, u]$. Since $f$ is $C^1$ by Assumption~\ref{ass:sshape}, the critical points of $f$ are the roots of $f'$ given by solving
    \[
    f'(\preact) = \ua - \sigma'(\preact)=0.
    \]
    By Lemma~\ref{lemma:two_tangency_points}, the only such solutions are $z_1,z_2$, coinciding when $\ua=B$. By the Interior Extremum Theorem, the minimum of $f$ on $[l, u]$ is attained either at an endpoint or at a critical point in $(l, u)$. It therefore suffices to show $f \ge 0$ at $l$, $u$, and at each critical point that falls inside $[l, u]$.

    \begin{itemize}
        \item Endpoints: $f(l), f(u) \ge 0$ by condition~(i).
        \item Left critical point: $f(z_1) \ge 0$ by condition~(ii).
        \item Right critical point: if $z_2\in [l,u]$:
        \begin{itemize}
            \item if $z_2 = u$, $f(z_2) = f(u) \ge 0$ by condition~(i).
            \item if $z_2 < u$, condition~(iii) implies $f(z_2)\geq 0.$
        \end{itemize} 
    \end{itemize}
    In all cases $f \ge 0$ on $[l, u]$, so $h_U \ge \sigma$ on $[l, u]$.
\end{proof}
\begin{remark}
Many activations, such as sigmoid, tanh, and arctan, satisfy an additional property: their derivative $\sigma'$ is an even function. For these functions $z_1=-z_2$, so Algorithm~\ref{alg:gadget} can check one derivative equality and the relation $z_1=-z_2$, retaining both activation-value checks.
\end{remark}

\textbf{Theorem~\ref{thm:casefreefourpoint}, part (b) (Case-free four-point completeness).}
\label{thm:casefreecompleteness}
\textit{Under Assumption~\ref{ass:sshape}, the \base algorithm selects linear relaxations $h_L$ and $h_U$ that satisfy the criteria of part~(a).}
\begin{proof}
    We prove the upper-bound case, and the lower bound is analogous. The selection procedure for $h_U(z)=\ua z+\ubof$ is described in Appendix~\ref{app:crown-details:relaxation}. We verify that its singleton case and each of its three remaining cases produce an upper bound $h_U$ satisfying conditions~(i), (ii), and (iii) of Theorem~\ref{thm:casefreefourpoint} part~(a). Its tangent slopes lie in $(0,B]$, since $0<\sigma'(z)\le B$ at every finite point. Its secant slopes lie in $(0,B)$, since they average this derivative over a nontrivial interval and $\sigma'(z)<B$ except at zero. Throughout, let $z_1\le 0\le z_2$ denote the tangency points satisfying $\sigma'(z_1)=\sigma'(z_2)=\ua$.

    For $l<0$, let $\tau>0$ denote the unique positive root of
    \[
    \sigma'(\tau)=\frac{\sigma(\tau)-\sigma(l)}{\tau-l},
    \]
    as established by Lemma~\ref{lemma:tangent-secant-root}. At this point, the tangent line of $\sigma(z)$ coincides with the secant line through $(l,\sigma(l))$ and $(\tau,\sigma(\tau))$. In the cases with $l<0$, we will use that $l<0<\tau$ and that $\tau$ maximizes the secant slope from $l$.

    If $l=u=c$, the selected tangent has zero residual at $c$.
    For $c\le0$, $z_1=c$ and
    \[
      h_U(z_2)-\sigma(z_2)=\int_{z_1}^{z_2}(\ua-\sigma'(z))\,dz\le0.
    \]
    Hence the product condition holds because $u-z_2\le0$.
    For $c\ge0$, $z_2=c$, the product is zero, and
    \[
      h_U(z_1)-\sigma(z_1)=\int_{z_1}^{z_2}(\sigma'(z)-\ua)\,dz\ge0.
    \]
    These arguments include $c=0$. For $l<u$, we check the three cases.

    \begin{itemize}
        \item \textbf{Case 1: $l<0<\tau<u$.} The procedure outputs the line
        \[
        h_U(z)=\sigma'(\tau)(z-\tau)+\sigma(\tau),
        \]
        which is simultaneously tangent to $\sigma$ at $z=\tau$ and is the secant connecting $(l,\sigma(l))$ to $(\tau,\sigma(\tau))$. It follows that $\ua=\sigma'(\tau)$ and $z_2=\tau$, where we use that $\tau>0$ since the tangent-secant point lies in the concave region.

        \emph{Condition~(i):} The left endpoint satisfies $h_U(l)=\sigma(l)$. For the right endpoint, $\sigma(u)\leq h_U(u)$ since $\sigma(\tau)=h_U(\tau)$ for $\tau\in(l,u)$ and $\sigma$ is concave on $[\tau,u]\subset [0,\infty)$.

        \emph{Condition~(ii):} Note that $\sigma'(z) \ge \ua$ for all $z \in [z_1, z_2]$. It follows that
        \[
        \sigma(z_1) = \sigma(z_2) - \int_{z_1}^{z_2} \sigma'(z) dz \le \sigma(z_2) - \int_{z_1}^{z_2} \ua dz = \sigma(z_2) - \ua(z_2 - z_1).
        \]
        Since $h_U$ has slope $\ua$ and $h_U(z_2) = \sigma(z_2)$, the right-hand side is equal to
        \[
        h_U(z_2) - \ua(z_2 - z_1) = h_U(z_1).
        \]
        Thus, $h_U(z_1) \ge \sigma(z_1)$.

        \emph{Condition~(iii):} This follows since $h_U(z_2)-\sigma(z_2)=0$.

        \item \textbf{Case 2: $l\ge0$.} The procedure outputs the tangent line at the midpoint $\mu:=(l+u)/2$:
        \[
        h_U(z)=\sigma'(\mu)(z-\mu)+\sigma(\mu).
        \]
        Since $[l,u]\subset [0,\infty)$, $\sigma$ is concave on $[l,u]$, so the tangent at any point within this interval dominates $\sigma$ on the entire interval. It follows that $\ua=\sigma'(\mu)$ and $z_2=\mu$.

        \emph{Condition~(i):} Since $\sigma$ is concave on $[l,u]$, the tangent line $h_U$ at $\mu\in[l,u]$ satisfies $h_U(l)\ge \sigma(l)$ and $h_U(u)\ge \sigma(u)$.

        \emph{Condition~(ii):} As in Case~1, $h_U$ is tangent to $\sigma$ at $z_2$ so $h_U(z_2)=\sigma(z_2)$. The identical integral argument over $[z_1,z_2]$ yields $h_U(z_1)\ge \sigma(z_1)$.

        \emph{Condition~(iii):} This follows since $h_U(z_2)-\sigma(z_2)=0.$

        \item \textbf{Case 3: $u\le0$, or $l<0<u\le\tau$.} In either case $l<0$ and $\tau\ge u$. The procedure outputs the secant line through $(l,\sigma(l))$ and $(u,\sigma(u))$:
        \[
        h_U(z)=\frac{\sigma(u)-\sigma(l)}{u-l}(z-l)+\sigma(l).
        \]
        We first observe that $z_2\ge u$. Because $\tau$ maximizes the secant slope from $l$, we have
        \[
        \sigma'(\tau)=\frac{\sigma(\tau)-\sigma(l)}{\tau-l}\geq \frac{\sigma(u)-\sigma(l)}{u-l}=\ua=\sigma'(z_2).
        \]
        Since $\tau, z_2 \in [0,\infty)$ and $\sigma'$ is strictly decreasing over this interval, $\sigma'(\tau)\ge \sigma'(z_2)$ implies $z_2\ge \tau\ge u$.

        \emph{Condition~(i):} $h_U(l)=\sigma(l)$ and $h_U(u)=\sigma(u)$ by construction.

        \emph{Condition~(ii):} Note that $z_1,l\le 0$, so $\sigma'$ is increasing between $l$ and $z_1$. Thus, if $z_1\le l$ then $\sigma'(z)\ge \ua$ for $z\in[z_1,l]$, and if $z_1>l$ then $\sigma'(z)\le \ua$  for $z\in[l,z_1].$ Either way,
        \[
        \sigma(z_1) = \sigma(l)-\int_{z_1}^{l}\sigma'(z)dz \le \sigma(l)-\int_{z_1}^{l}\ua dz=\sigma(l)-\ua(l-z_1)=h_U(z_1).
        \]
        \emph{Condition~(iii):} Since $z_2\ge u$, the first factor satisfies $u-z_2\le 0$. We show that the second factor satisfies $h_U(z_2)-\sigma(z_2)\le 0$. To see this, note that $\sigma'(z)\ge \ua$ for all $z\in [z_1,z_2]$, and that the Mean Value Theorem gives a derivative root in $(l,u)$. Since $z_2\ge u$, this root is $z_1$, implying that $\sigma'(z)\ge \ua$ for all $z\in[u,z_2].$ Integrating yields
        \[
        \sigma(z_2)=\sigma(u)+\int_{u}^{z_2}\sigma'(z) dz\ge \sigma(u)+\int_{u}^{z_2}\ua dz=\sigma(u)+\ua(z_2-u).
        \]
        Since $h_U$ has slope $\ua$ and $h_U(u)=\sigma(u)$, the right-hand side is equal to
        \[
        h_U(u)+\ua(z_2-u)=h_U(z_2).
        \]
        Thus, $\sigma(z_2)\ge h_U(z_2).$
    \end{itemize}
    For the lower bound apply the upper result to
    $\widetilde\sigma(x)=-\sigma(-x)$ on $[-u,-l]$.
    Reflection exchanges the two critical points and transforms the upper
    conditions into the lower conditions of part~(a).
\end{proof}

\end{document}

\newpage
\section*{NeurIPS Paper Checklist}

\begin{enumerate}

\item {\bf Claims}
    \item[] Question: Do the main claims made in the abstract and introduction accurately reflect the paper's contributions and scope?
    \item[] Answer: \answerYes{} 
    \item[] Justification: As stated in Section~\ref{sec:evaluation}, we distinguish the proposed ZKP system, \tool, from the evaluated implementation, which uses non-zero-knowledge variants of the underlying primitives. The abstract, introduction, and Section~\ref{sec:evaluation} accurately attribute the reported performance to this implementation. In both the abstract and the introduction, we only state that \tool protects model parameters. The full scope of \tool's guarantees, including quantization and disclosure of the model architecture, is discussed in Section~\ref{sec:limitations}.

\item {\bf Limitations}
    \item[] Question: Does the paper discuss the limitations of the work performed by the authors?
    \item[] Answer: \answerYes{} 
    \item[] Justification: The main known limitations of the approach are discussed in Section~\ref{sec:limitations}.

\item {\bf Theory assumptions and proofs}
    \item[] Question: For each theoretical result, does the paper provide the full set of assumptions and a complete (and correct) proof?
    \item[] Answer: \answerYes{} 
    \item[] Justification: Proofs of all results are given in the main text or the Appendix.

    \item {\bf Experimental result reproducibility}
    \item[] Question: Does the paper fully disclose all the information needed to reproduce the main experimental results of the paper to the extent that it affects the main claims and/or conclusions of the paper (regardless of whether the code and data are provided or not)?
    \item[] Answer: \answerYes{} 
    \item[] Justification: We provide the PANDA implementation, floating-point CROWN baseline, benchmark-preparation scripts, evaluation parameters, and execution instructions at \url{https://github.com/youweizhong/PANDA}. The experimental setup is described in Section~\ref{sec:evaluation} and Appendix~\ref{comparison}. Our FairProof reimplementation is available upon request from \href{mailto:youwei.zhong@yale.edu}{youwei.zhong@yale.edu}. 

\item {\bf Open access to data and code}
    \item[] Question: Does the paper provide open access to the data and code, with sufficient instructions to faithfully reproduce the main experimental results, as described in supplemental material?
    \item[] Answer: \answerYes{} 
    \item[] Justification: The measured PANDA prototype, CROWN baseline, benchmark-preparation scripts, and reporting tools are public at \url{https://github.com/youweizhong/PANDA}, with instructions for obtaining external assets. The FairProof reimplementation is excluded because the upstream repository provides no explicit license, but is available upon request from \href{mailto:youwei.zhong@yale.edu}{youwei.zhong@yale.edu}. The FairProof Adult example is available upstream.

\item {\bf Experimental setting/details}
    \item[] Question: Does the paper specify all the training and test details (e.g., data splits, hyperparameters, how they were chosen, type of optimizer) necessary to understand the results?
    \item[] Answer: \answerYes{} 
    \item[] Justification: Section~\ref{sec:evaluation} specifies benchmark selection, quantization precision, and hardware for the reported prototype; the repository at \url{https://github.com/youweizhong/PANDA} provides execution details.

\item {\bf Experiment statistical significance}
    \item[] Question: Does the paper report error bars suitably and correctly defined or other appropriate information about the statistical significance of the experiments?
    \item[] Answer: \answerYes{} .
    \item[] Justification: Table~\ref{final_least_table} reports means and standard deviations of prototype runtimes. FairProof Adult has one test and no reported standard deviation.

\item {\bf Experiments compute resources}
    \item[] Question: For each experiment, does the paper provide sufficient information on the computer resources (type of compute workers, memory, time of execution) needed to reproduce the experiments?
    \item[] Answer: \answerYes{} 
    \item[] Justification: All details on compute resources are provided in Section~\ref{sec:evaluation}.
    
\item {\bf Code of ethics}
    \item[] Question: Does the research conducted in the paper conform, in every respect, with the NeurIPS Code of Ethics \url{https://neurips.cc/public/EthicsGuidelines}?
    \item[] Answer: \answerYes{} 
    \item[] Justification: Our experiments use previously published benchmark assets and do not involve human-subject recruitment or new personal-data collection. Appendix~\ref{sec:societal-impact} discusses potential benefits, limitations, and misuse risks.

\item {\bf Broader impacts}
    \item[] Question: Does the paper discuss both potential positive societal impacts and negative societal impacts of the work performed?
    \item[] Answer: \answerYes{} 
    \item[] Justification: Appendix~\ref{sec:societal-impact} discusses the potential benefits of publicly verifiable model guarantees, as well as risks from overinterpreting local robustness as general fairness or safety, reduced visibility into private models, and misuse of the non-zero-knowledge prototype.
    
\item {\bf Safeguards}
    \item[] Question: Does the paper describe safeguards that have been put in place for responsible release of data or models that have a high risk for misuse (e.g., pre-trained language models, image generators, or scraped datasets)?
    \item[] Answer: \answerNA{} 
    \item[] Justification: The work uses previously published benchmark assets and releases no new trained models. Section~\ref{sec:evaluation} identifies the non-zero-knowledge implementation, and Appendix~\ref{sec:societal-impact} discusses risks from treating its output as a deployment-level guarantee.

\item {\bf Licenses for existing assets}
    \item[] Question: Are the creators or original owners of assets (e.g., code, data, models), used in the paper, properly credited and are the license and terms of use explicitly mentioned and properly respected?
    \item[] Answer: \answerYes{} 
    \item[] Justification: We credit the original creators of the code, libraries, datasets, and benchmarks used in this work. Our implementation of PANDA carries an MIT license. We provide source URLs, known licensing terms, and links to third-party Rust and Python library licenses in \url{https://github.com/youweizhong/PANDA/blob/main/THIRD_PARTY_CREDITS.md}.

\item {\bf New assets}
    \item[] Question: Are new assets introduced in the paper well documented and is the documentation provided alongside the assets?
    \item[] Answer: \answerYes{} 
    \item[] Justification: The reported prototype and evaluation tools are documented at \url{https://github.com/youweizhong/PANDA}. The FairProof reimplementation described in Appendix~\ref{comparison} is excluded from the public release because the upstream repository provides no explicit license, but is available upon request from \href{mailto:youwei.zhong@yale.edu}{youwei.zhong@yale.edu}.

\item {\bf Crowdsourcing and research with human subjects}
    \item[] Question: For crowdsourcing experiments and research with human subjects, does the paper include the full text of instructions given to participants and screenshots, if applicable, as well as details about compensation (if any)? 
    \item[] Answer: \answerNA{} 
    \item[] Justification: This work does not involve any crowdsourcing or research with human subjects.

\item {\bf Institutional review board (IRB) approvals or equivalent for research with human subjects}
    \item[] Question: Does the paper describe potential risks incurred by study participants, whether such risks were disclosed to the subjects, and whether Institutional Review Board (IRB) approvals (or an equivalent approval/review based on the requirements of your country or institution) were obtained?
    \item[] Answer: \answerNA{} 
    \item[] Justification: This work does not involve any crowdsourcing or research with human subjects.

\item {\bf Declaration of LLM usage}
    \item[] Question: Does the paper describe the usage of LLMs if it is an important, original, or non-standard component of the core methods in this research? Note that if the LLM is used only for writing, editing, or formatting purposes and does \emph{not} impact the core methodology, scientific rigor, or originality of the research, declaration is not required.
    \item[] Answer: \answerNA{} 
    \item[] Justification: This work does not involve LLMs as any important, original or non-standard component. LLMs are only used for minor writing assistance, all of which is rewritten and checked by a human. LLMs are also used for basic coding assistance.

\end{enumerate}

\end{document}